\documentclass[final,10pt]{elsarticle}
\usepackage[english]{babel} %

\usepackage[full]{textcomp}
\renewcommand{\copyright}{\textcopyright}

\usepackage{mathdesign}
\usepackage{multirow}
\usepackage{booktabs}
\usepackage{subcaption}
\usepackage[justification=centering]{caption}
\usepackage[dvipsnames]{xcolor}
\usepackage{pgfplots}
\usetikzlibrary{calc}

\usepackage{makecell, cellspace, caption}

\usepackage{amssymb} %
\usepackage{amsfonts}
\usepackage{amsmath}
\usepackage{amsthm}
\usepackage{amsmath,amsfonts,amsthm,bm} %
\usepackage{hyperref} %
\hypersetup{
        colorlinks
}
\usepackage{algorithm} %
\usepackage{algpseudocode} %

\usepackage[switch,displaymath]{lineno}

\usepackage{graphicx}
\usepackage{times}
\usepackage{epsfig}
\usepackage{tabularx}
\usepackage{mathtools}
\usepackage{color, colortbl}
\usepackage{enumitem}
\usepackage{hhline}
\usepackage{bbm}
\usepackage{wrapfig}

\usepackage{amsmath,amsfonts,bm}

\def\eqref#1{equation~\ref{#1}}

\def\1{\bm{1}}

\DeclareMathAlphabet{\mathsfit}{\encodingdefault}{\sfdefault}{m}{sl}
\SetMathAlphabet{\mathsfit}{bold}{\encodingdefault}{\sfdefault}{bx}{n}

\newcommand{\ceq}{\stackrel{\mathclap{\normalfont\mbox{c}}}{=}}

\definecolor{LightCyan}{rgb}{0.88,1,1}
\definecolor{LightGray}{rgb}{0.9,0.9,0.9}
\definecolor{Gray}{gray}{0.9}

\newtheorem{proposition}{Proposition}

\journal{Medical Image Analysis}

\title{Do We Really Need Dice? The Hidden Region-Size Biases of Segmentation Losses}

\author[1,2]{Bingyuan Liu\corref{cor1}}
\cortext[cor1]{Corresponding author: \url{liubingyuan1988@gmail.com}}

\author[1,2,3]{Jose Dolz}

\author[4]{Adrian Galdran}

\author[5]{Riadh Kobbi}

\author[1,2,3]{Ismail Ben Ayed}

\address[1]{LIVIA, ÉTS Montréal, Canada}
\address[2]{International Laboratory on Learning Systems (ILLS),\\McGill - ETS - MILA - CNRS - Université Paris-Saclay - CentraleSupélec, Canada}
\address[3]{Centre de Recherche du Centre Hospitalier de l'Université de Montréal (CRCHUM), Canada}
\address[4]{Universitat Pompeu Fabra, Spain}
\address[5]{Diagnos Inc., Canada}

\newcommand{\rev}[1]{{\color{black}#1}}
\newcommand{\revA}[1]{{\color{black}#1}}

\begin{document}
\begin{frontmatter}
\begin{abstract}

Most segmentation losses are arguably variants of the Cross-Entropy (CE) or Dice losses. On the surface, these two categories of losses (i.e., distribution based vs. geometry based) seem unrelated, 
and there is no clear consensus as to which category is a better choice, with varying performances for each across different benchmarks and applications. Furthermore, it is widely argued within the 
medical-imaging community that Dice and CE are complementary, which has motivated the use of compound CE-Dice losses.
In this work, we provide a theoretical analysis, which shows that CE and Dice share a much deeper connection than previously thought. First, we show that, from a constrained-optimization perspective, they both decompose into two components, i.e., a similar ground-truth matching term, which pushes the predicted foreground regions towards the ground-truth, and a region-size penalty term imposing different biases on the size (or proportion) of the predicted regions.
Then, we provide bound relationships and an information-theoretic analysis, which uncover hidden region-size biases: Dice has an intrinsic bias towards 
specific extremely imbalanced solutions, whereas CE implicitly encourages the ground-truth region proportions. Our theoretical results explain the wide experimental evidence in the medical-imaging literature, whereby
Dice losses bring improvements for imbalanced segmentation.
It also explains why CE dominates natural-image problems with diverse class proportions, in which case Dice might have difficulty adapting to different region-size distributions.
Based on our theoretical analysis, we propose a principled and simple solution, which enables to control explicitly the region-size bias.
The proposed method integrates CE with explicit terms based on ${\cal L}_1$ or the KL divergence, which encourage segmenting region proportions to match target class proportions, thereby mitigating class imbalance but without losing generality. Comprehensive experiments and ablation studies over different losses and applications validate our theoretical analysis, as well as the effectiveness of explicit and simple region-size terms.
The code is available at \url{https://github.com/by-liu/SegLossBias}.
\end{abstract}


\end{frontmatter}

\setlength{\parskip}{3pt}

\section{Introduction}
\label{sec:intro}

Semantic segmentation is one of the most investigated problems in computer vision, and has been impacting a breadth of applications, from natural-scene 
understanding \citep{Cordts2016Cityscapes,Kirillov2019Panoptic,kirillov2023segment} to medical image analysis \citep{Litjens2017survey,DOLZ2018segmri,cheng2023sammed2d}. 
The problem is often stated as pixel-wise classification, following the optimization of a loss function expressed with summations over the 
ground-truth regions, as in the standard Cross-Entropy (CE) loss.
A challenging aspect of segmentation problems is the existence of extremely diverse distributions (or proportions of the segmentation regions) across different datasets, classes and instances.
A representative example is the popular Cityscapes dataset \citep{Cordts2016Cityscapes}, where the average proportions of some classes, such as \textit{motorcycle} or \textit{bicycle}, are below $1\%$, while the proportions of some classes, like \textit{road} and \textit{building}, can be larger than $10\%$.
The class imbalance issue in medical image segmentation can be even more severe.
It usually involves medium-to-large regions like liver or pancreas, and small regions such as tumor \citep{antonelli2021medical}.
In some situations, the examples might have extremely small regions like in the context of retinal lesions \citep{Wei2020RetinalLesion}.
Therefore, segmentation methods should be able to address the extreme class imbalance issue (small-region terms are nearly neglected in the objective), without losing generality to adapt to medium-to-large regions.  
In these scenarios, besides specifically designed deep-network architectures or training schemes \citep{tao2020hierarchical,beit2021BEiT}, the loss function to be minimized during learning plays a 
critical role, and has triggered a large body of research works in the recent years \citep{ma2021loss,KERVADEC2021Boundary,Lin2017FocalLoss,WongMTS18logdice,vnet2016,Carole2017GDice,kervadec2021beyond,Kofler2022Blob}.

\begin{figure}[t]
\centering
\includegraphics[width=0.8\columnwidth]{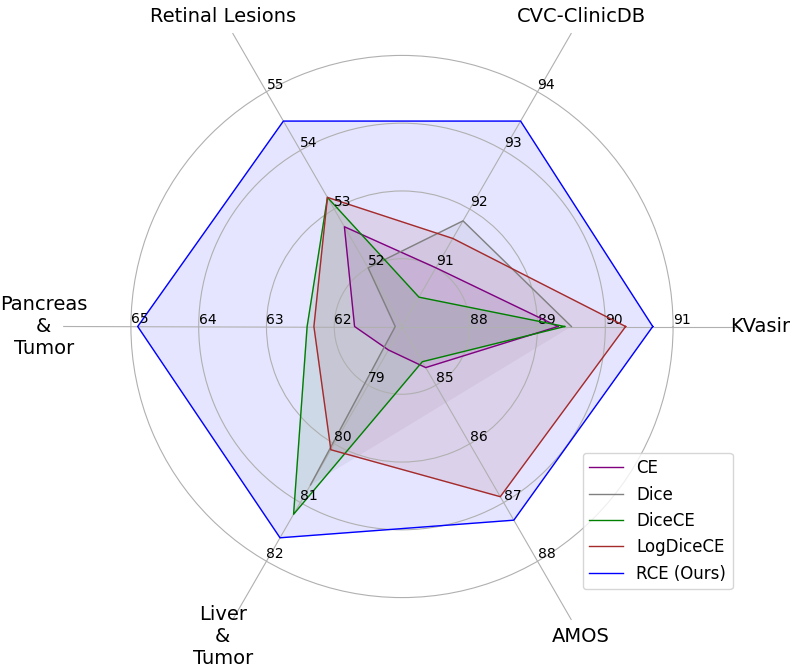}
\caption{Comparison of the proposed loss for semantic segmentation, RCE, with the widely used losses of CE, Dice, and their compound forms (DiceCE and LogDiceCE) \revA{across six medical image segmentation benchmarks}. The DSC($\%$) scores achieved with R50FPN segmentation network on the test set of each database are included.}
\label{fig:radar}
\end{figure}

While there exists a great diversity of loss functions for segmentation, the recent excellent survey in \citep{ma2021loss} pointed to strong connections between 
these losses; see Fig. 1 in \citep{ma2021loss}. Most of the existing segmentation losses are arguably variants of the CE, Dice loss \citep{ma2021loss,yeung2021mixed} or combinations of both \citep{WongMTS18logdice,taghanaki2019combo}, and could be categorized into two main families. The first family is motivated by distribution measures, i.e., CE and its variants, and is directly adapted 
from classification tasks. To deal with class imbalance, various extensions of CE have been investigated, such as increasing the relative weights for minority classes \citep{Ronn2015Unet},
or modifying the loss so as to account for performance indicators during training, as in the popular Focal loss \citep{Lin2017FocalLoss} or TopK loss \citep{WuSH16topk}.
The second main family of losses is inspired by geometrical metrics. In this category, the most popular losses are linear Dice \citep{vnet2016} and its extensions, such
as the logarithmic \citep{WongMTS18logdice} or generalized \citep{Carole2017GDice} Dice loss. Borrowing the idea of the weighted CE, the latter introduces class weights to increase the contributions of 
the minority classes. These loss functions are motivated by the geometric Dice coefficient, which measures the overlap between the ground-truth and predicted segmentation regions.

In the literature, to our knowledge, there is no clear consensus as to which category of losses is better, with the performances of each 
varying across data sets and applications. It has been empirically argued that the Dice loss and its variants are more appropriate for extreme class imbalance, and such empirical 
observations are the main motivation behind the wide use and popularity of Dice in medical-imaging applications \citep{vnet2016,jha2020doubleunet}. 
CE, however, dominates most recent models in the context of natural images \citep{ChenZPSA18deeplab,zhao2017pspnet,zhao2018psanet,Yuan2020OCR}, and also outperforms Dice based losses for some benchmarks or categories in medical imaging \citep{ma2021loss}.
Therefore, beyond experimental evidence, there is a need for a theoretical analysis that clarifies which segmentation loss to adopt for a given task, a decision that may affect performance significantly.

On the surface, these two categories of losses (i.e., distribution based vs. geometry based) seem unrelated.
Moreover, it is widely argued within the medical imaging community that Dice and CE are complementary losses, which has motivated the use of compound CE-Dice losses integrating both \citep{ma2021loss,Fabian2021nnUnet,WongMTS18logdice,taghanaki2019combo}. 
These recent works, among many others, provided an intensive experimental evidence that points to highly competitive performances of compound CE-Dice losses,  
in a variety of class-imbalance scenarios. In particular, the recent comprehensive experimental study in \citep{ma2021loss} corroborated this finding with 
evaluations over more than 20 recent segmentation losses.

In this paper, we provide a constrained-optimization perspective showing that, in fact, CE and Dice share a much deeper connection 
than previously thought: They both decompose into region-size penalties and closely related ground-truth matching penalties.  
Our theoretical analysis highlights encoded hidden region-size biases in Dice and CE, and shows that the main difference between 
the two types of losses lies essentially in those region-size biases: Dice has an intrinsic bias preferring very small regions, while 
CE implicitly encourages the right (ground-truth) region proportions. Our results explain the wide experimental evidence in 
the medical-imaging literature, whereby using or adding Dice losses brings improvements for imbalanced segmentation with extremely small regions.  
It also explains why CE dominates natural-image problems and has edges in some medical imaging applications with diverse class proportions, in which case Dice might have difficulty adapting to 
different region-size distributions (see examples in Fig.~\ref{fig:example}).
Based on our theoretical analysis, we propose principled and simple loss functions, which enable to control explicitly the region-size bias term. 
Our solution integrates the benefits of both categories of losses, mitigating class imbalance but without losing generality, showing competitive and more stable performances on a variety of bechmarks, as shown in Fig.~\ref{fig:radar}.

Our contributions are summarized as follows:

    $\bullet~$ Showing through an explicit bound relationship (Proposition \ref{prop:multi-class-Dice-Bias}) that the Dice loss has a hidden region-size bias towards 
    specific extremely imbalanced solutions, preferring small structures, while losing the flexibility to deal effectively with arbitrary class proportions.
    
    $\bullet~$ Providing an information-theoretic perspective of CE, via Monte-Carlo approximation of the entropy of the learned features (Proposition \ref{prop:ce}). 
    This highlights a hidden region-size bias of CE, which encourages the proportions of the predicted segmentation regions to match the ground-truth proportions.
    
    $\bullet~$ Introducing compound loss functions, which enables to control explicitly class-proportion biases 
    in standard supervised-learning settings: 
    Our losses integrate CE with explicit terms based on ${\cal L}_1$ or the KL 
    divergence, which encourage segmenting regions to match target class proportions.  
    
    $\bullet~$ Comprehensive  experiments and ablation studies over different losses and applications, including 2d and 3d medical-imaging data, validate our theoretical analysis, as well as the effectiveness of simple region-size regularizers.

\revA{
\section{Related work}

\textbf{Semantic Segmentation.} Before the emergence of deep learning, popular models for image segmentation included graph-based approaches \citep{BoykovF06GrapCut, Tang2016MRF} and conventional machine learning techniques, e.g., support vector machines, with hand-designed descriptors \citep{Dong2014Semantic}.
Current state-of-the-art models in both natural image benchmarks and medical domain are based on an encoder-decoder architecture \citep{Badrinarayanan2017SegNet}.
In this type of architecture, the encoder typically leverages some off-the-shelf CNN architecture like Residual Networks \citep{He2016resnet}, or self-attention based transformer \citep{dosovitskiy2020vit} pre-trained on a large-scale supervised dataset \citep{ILSVRC15,Lin14Coco}.
The encoder is then followed by a decoder that semantically projects the discriminative features learned by the encoder onto the pixel space to obtain a dense classification as the predicted mask.
A large body of research in image segmentation has focused on improving the design of the decoder module with techniques like feature pyramid (FPN),  Atrous convolution (DeepLab) \citep{ChenZPSA18deeplab}, attention (PSA)\citep{zhao2018psanet}, transformer \citep{Vaswani2017transformer} and many other alternatives \citep{Zbigniew2018Devil}.
In the medical image applications, the U-Net \citep{Ronn2015Unet} and its variants \citep{vnet2016, galdran2020wnet, hatamizadeh2022unetr} have become the most popular options for image segmentation.
Recently, we also see the efforts to build foundational segmentation models \citep{kervadec2021beyond, cheng2023sammed2d}, with strong generalization ability and zero-shot performances. 

\textbf{Segmentation Losses.} Once a model structure has been selected, the next critical decision to make is the loss function to be minimized.
On the basis of motivations, most loss functions can be mainly categorized into two groups.
The first group is motivated by statistical metric, such as Cross-entropy (CE), which is the most widely used option.
Under the presence of class imbalance, some simple extensions of CE may be preferred, \textit{e.g.}, weighting different classes according to the corresponding inverse class frequencies \citep{Carole2017GDice}.
Some other works use performance indicators to dynamically increase the attention to hard or minority examples during training, such as Focal loss \citep{Lin2017FocalLoss} and TopK loss \citep{WuSH16topk}.

The second family of losses is recognized as geometrical-based functions.
Inspired by Dice coefficient, Dice loss, first applied in \citep{vnet2016}, is a popular alternative for CE, especially in the medical image segmentation community.
It has the advance of directly maximizing the evaluation metric and handling highly imbalance issues due to its sensitivity of few misclassified pixels.
Like weighted CE, generalised Dice loss \citep{Carole2017GDice} imposes the inverse of the area as the class weights.
With the standard Dice loss set as $1-Dice$, \citep{WongMTS18logdice} propose to use the exponential logarithmic form of Dice as an alternative.
Some works try to extend Dice-based losses with particular motivations.
clDice \citep{cldice2021} considering the intersection of the segmentation masks and their morphological skeleta.
Blob loss \citep{Kofler2022Blob} including instance imbalanced awareness to improve the performance of segmenting multiple small instances.
\citep{wang2023dice} adapts the Dice loss for soft labels.
Some recent works \citep{WongMTS18logdice,taghanaki2019combo,Yue2019CardiacSeg} claim the benefit of combining CE and logarithmic Dice, which is also referred as a compound loss.
Besides Dice metric, alternative geometric measures have driven the exploration of other segmentation losses, such as Tversky loss \citep{salehi2017tversky}, Jaccard or Intersection over Union (IoU) \citep{Eelbode2020OptJaccard,duquearias2021JaccardLoss}, and boundary loss \citep{KERVADEC2021Boundary}.
In \citep{kervadec2021beyond}, it demonstrates that comparable performances can be achieved with global geometric shape descriptors only, without the standard pixel-wise cross-entropy loss.

In the context of the large body of segmentation losses proposed in recent years,  there is no consensus on how to chose a good and appropriate loss.
\citep{ma2021loss,yeung2021mixed} comprehensively explored the links and differences between different losses, while \citep{leng2022polyloss} propose a unified framework for common losses by Taylor expansion.
In this paper, we further reveal a non-obvious relationship between both two main families of segmentation losses, \textit{i.e.}, CE and Dice, via an explicit theoretical justification.
}

\section{Formulation}
\label{sec:method}

\begin{table*}[!h]
    \caption{\textbf{Notations, formulations and approximations used in this paper.} ${\cal F}$ and ${\cal K}$ denotes the random variables associated with the learned features and the labels, respectively. ${\mathbb P}$ denotes probability.  $|.|$ denotes cardinality when the input is a set and the standard absolute value when the input is a scalar. Note that network parameters $\theta$ are omitted in the prediction quantities, so as to simplify notations, 
    as this does not lead to ambiguity.}
    \label{table:notations}
    \centering
    \resizebox{1.0\textwidth}{!}{
    \renewcommand{\arraystretch}{1.2}
    \begin{tabular}{@{}ll@{}}
    \multicolumn{2}{c}{Dataset} \\
    \toprule
    Concept & Formula \\
    \midrule
    Indices/number of classes & $1 \leq k \leq K$
    \\
     Spatial image domain & $\mathbf{\Omega}\subset \mathbb{R}^2$ 
    \\
    Labels of pixel $i \in \mathbf{\Omega}$ & $y_{ik} \in \{0, 1\}$ \\
    GT region $k$ & $\mathbf{\Omega}_k = \{i \in \mathbf{\Omega} | y_{ik} = 1\}$ \\
     GT proportion of region $k$ & $\hat{y}_{k} = \frac{|\mathbf{\Omega}_k|}{|\mathbf{\Omega}|}$
    \\
    GT region-size prob. & $\mathbf{y} = \left (\hat{y}_{k} \right )_{1 \leq k \leq K}$ 
    \\
    \bottomrule
    \end{tabular}
    \qquad
    \begin{tabular}{@{}ll@{}}
    \multicolumn{2}{c}{Modeling} \\
    \toprule
    Concept & Formula \\
    \midrule
    Model parameters & $\theta$ 
    \\
    Feature embedding at pixel $i \in \mathbf{\Omega}$
      & $\mathbf{f}_i^{\theta}$ \\
    Softmax predictions at pixel $i \in \mathbf{\Omega}$  & $p_{ik} = \mathbb{P}(k| \mathbf{f}_i^{\theta})$
    \\
    Predicted proportion of class $k$  & $\hat{p}_{k} = \frac{1}{\mathbf{|\Omega|}}\sum_{i \in \mathbf{\Omega}} p_{ik}$ \\
    Predicted region-size prob. & $\mathbf{p} = \left (\hat{p}_{k} \right )_{1 \leq k \leq K}$  \\
    ($K-1$)-simplex & $\Delta_K = \{\mathbf{p} \in [0, 1]^K~/~\sum_k \hat{p}_{k} = 1 \}$ \\
    \bottomrule \\
    \end{tabular}
    }
    \resizebox{0.7\textwidth}{!}{
    \renewcommand{\arraystretch}{1.5}
    \begin{tabular}{@{}lcl@{}}
    \multicolumn{3}{c}{\rule{0pt}{4ex}Losses, region-size regularizers and information-theoretic quantities} \\
    \toprule
    Concept && Formula \\
    \midrule
    Weighted cross-entropy && $ \text{CE} = -\sum_{k=1}^K \frac{1}{|\mathbf{\Omega}_k|}\sum_{i \in \mathbf{\Omega}_k} \log(p_{ik})$ \\
    Dice coefficient for region $k$ && $ \text{Dice}_k = \frac{2 \sum_{i\in \mathbf{\Omega}_k} p_{ik} }{\sum_{i \in \mathbf{\Omega}}{p_{ik}^{}} + |\mathbf{\Omega}_k|} $  \\
    region-size KL divergence && $ {\cal D}_\text{KL}(\mathbf{{y}}||\mathbf{p}) =  \sum_{k=1}^{K} \hat{y}_{k} \log(\frac{\hat{y}_{k}}{\hat{p}_{k}})$  \\
    region-size ${\cal L}_1$ distance &&  ${\cal L}_1 (\mathbf{{y}}, \mathbf{p}) = \sum_{k=1}^{K} | \hat{y}_{k} - \hat{p}_{k} | $  \\
    \makecell[l]{{\em Monte-Carlo} estimate of the \\entropy of features given region $k$}
    && $ {\cal H}({\cal F}|{\cal K} = k) \approx - \frac{1}{|\mathbf{\Omega}_k|} \sum_{i \in \mathbf{\Omega}_k} \log(\mathbb{P}(\mathbf{f}_i^{\theta}|k))$ \\
    \bottomrule
    \end{tabular}
    }
\end{table*}

Semantic segmentation is often stated as a pixel-wise classification task, following the optimization of a loss function for training a deep network.
\rev{Specifically, training dataset is defined as ${\cal D} = \{(X_n, Y_n)\}_n$. An input image is $X_n: \mathbf{\Omega} \rightarrow \mathbb{R}^2$, where $\mathbf{\Omega}\subset \mathbb{R}^2$ denotes the spatial image domain, and the corresponding ground truth (GT) is $Y_n: \mathbf{\Omega} \rightarrow \{0, 1\}^K$ where K is the number of classes.
We denote the region $k$ in GT as $\mathbf{\Omega}_k = \{i \in \mathbf{\Omega} | y_{ik} = 1\}$, and then the proportion of region $k$ is calculated as $\hat{y}_{k} = \frac{|\mathbf{\Omega}_k|}{|\mathbf{\Omega}|}$.
Assume we have a deep network parameterized by $\theta$, it generates a feature embedding (or logit) for each pixel $\mathbf{f}_i^{\theta}$. Note that the feature embedding is the input of the softmax probability prediction of the network, which is denoted as $(p_{ik})_{1 \leq k \leq K}$.
With the softmax prediction, the predicted proportion of class $k$ is calculated as $\hat{p}_{k} = \frac{1}{\mathbf{|\Omega|}}\sum_{i \in \mathbf{\Omega}} p_{ik}$.
In Table \ref{table:notations}, we listed all the notations, formulations and approximations used in this paper.
Besides the basic notations of the task (such as networks predictions), we explicitly include the loss functions, region-size regularizers and 
information-theoretic quantities that will be discussed in the following sections.}
We note that, to facilitate the reading of our analysis, we write the CE and Dice losses in a non-standard way using summations over the ground-truth segmentation regions, rather than as functions of 
the labels. Also, while we provide the CE loss for all segmentation regions, we give Dice for a single region. This is to accommodate two 
variants of the Dice loss in the literature: in the binary case, Dice is typically used for the foreground region only \citep{vnet2016}; in the multi-region case, it is commonly 
used over all the regions \citep{WongMTS18logdice}. Finally, to simplify notation, we give all the loss functions for a single training image, without summations over all training samples (as this does not lead 
to any ambiguity, neither does it alter the analysis hereafter). In the training iterations, we use the mean values across all the training samples 
via standard mini-batch optimization. 

\subsection{Definition of region-size biases and penalty functions}
In the following, we analyse the region-size biases inherent to CE and Dice losses, and show that the main difference between the two types of losses lies essentially in 
those region-size biases. To do so, we provide a constrained-optimization perspective of the losses.
\rev{For the discussion, consider specifically the following hard equality constraint:
\begin{equation}
\label{eq:hard_equality}
\mathbf{p} = \mathbf{t}
\end{equation}
where $\mathbf{t}$ is a given (fixed) target distribution.
In the general context of constrained optimization, penalty functions are widely used \citep{Bertsekas95}. 
Then we define a region-size bias as a principled soft {\em penalty} function for the above hard equality: $g(\mathbf{p})$, which is added to the main objective (e.g., CE) being minimized to replace the hard equality constraint.
Following the general principle of a soft-penalty optimizer, the penalty function $g$ increases when $\mathbf{p}$ deviates from target $\mathbf{t}$.
}
By definition, for the constraint $\mathbf{p} = \mathbf{t}$, with the domain 
of $\mathbf{p}$ being probability simplex $\Delta_K$, a penalty $g(\mathbf{p})$ is a continuous and differentiable function, which reaches 
its global minimum when the constraint is satisfied, i.e., it verifies: $g(\mathbf{t}) \leq g(\mathbf{p})\, \forall \mathbf{p} \in \Delta_K$.

\subsection{The link between Cross Entropy and Dice}

To ease the discussion in what follows, we will start by analyzing the link between CE and the logarithmic Dice, along with the region-size 
bias of the latter (Proposition \ref{prop:multi-class-Dice-Bias}). Then, we discuss a bounding relationship between the different Dice variants. 
Finally, we will provide an information-theoretic analysis, which 
highlights the hidden region-size bias of CE (Proposition \ref{prop:ce}).

Let us consider the logarithmic Dice loss in the multi-class case. This loss decomposes (up to a constant) into two terms, a ground-truth matching term and a region-size bias:  
\begin{equation}
\label{eq:logdice-multiclass}
- \sum_{k=1}^K \log(\text{Dice}_k) \ceq \underbrace{- \sum_{k=1}^K \log\left (\frac{1}{|\mathbf{\Omega}_k|}\sum_{i\in \mathbf{\Omega}_k} p_{ik} \right )}_{\text{Ground-truth matching: DF}} + 
 \underbrace{\sum_{k=1}^K \log\left ( \hat{p}_{k} + \hat{y}_{k} \right )}_{\text{region-size bias: DB}}
\end{equation}
where $\ceq$ stands for equality up to an additive and/or non-negative multiplicative constant. 

\rev{

\begin{proposition}
\label{prop:df-bound-ce}
The ground-truth matching term in the logarithmic Dice (DF in Eq. (\ref{eq:logdice-multiclass})) is lower bounded on the cross-entropy loss (CE):

\begin{equation}
\label{eq:df-bound-ce}
DF = - \sum_{k=1}^K \log\left (\frac{1}{|\mathbf{\Omega}_k|}\sum_{i\in \mathbf{\Omega}_k} p_{ik} \right ) \leq CE
\end{equation}
\end{proposition}

}

\rev{The detailed proof is deferred to Appendix A.1, which is mainly due to Jensen's inequality and the convexity of function $-\log(x)$.}

Therefore, minimizing CE could be viewed as a proxy for minimizing term DF that appears in the logarithmic Dice. In fact, from a 
constrained-optimization perspective, DF and CE are very closely related and could be viewed as two different penalty functions enforcing the same equality constraints: $p_{ik} = 1, \, \forall i \in {\mathbf \Omega}_k, \, \forall k$.
Both DF and CE are monotonically decreasing functions of each softmax and reach their global minimum when these equality constraints are satisfied. Therefore, they encourage softmax predictions $p_{ik}$ for 
each region ${\mathbf \Omega}_k$ to reach their target ground-truth values of $1$. Of course, this does not mean that penalties CE and DF yield exactly the same results. The difference in the results that 
they may yield is due to the optimization technique (e.g., different gradient dynamics in the standard training of deep networks as the penalty functions have different forms).

\subsection{The hidden region-size bias of Dice}
\label{sec:dice_bias}

The following proposition highlights how the region-size term DB in Eq. (\ref{eq:logdice-multiclass}) encourages specific extremely imbalanced solutions. 
\begin{proposition}
\label{prop:multi-class-Dice-Bias}
Let $\mathbf{t}=\left (\hat{t}_{j} \right )_{1 \leq j \leq K} \in \{0, 1\}^K$ denote the simplex vertex verifying: $\hat{t}_{j} = 1$ when $\hat{y}_{j} = \max_{\tiny{1 \leq k \leq K}} \hat{y}_{k}$
and $\hat{t}_{j} = 0$ otherwise. For variables $\mathbf{p}=\left (\hat{p}_{k} \right )_{1 \leq k \leq K}$ and fixed distribution $\mathbf{y} = \left (\hat{y}_{k} \right )_{1 \leq k \leq K}$, the 
region-size term in Eq. (\ref{eq:logdice-multiclass}) reaches its minimum over the simplex at $\mathbf{t}$:    
\begin{equation}
\label{imbalanced-bias-Dice}
\sum_{k=1}^K \log\left ( \hat{t}_{k} + \hat{y}_{k}  \right ) \leq \sum_{k=1}^K \log\left ( \hat{p}_{k} + \hat{y}_{k}  \right ) \quad \forall \mathbf{p} \in \Delta_K 
\end{equation}
\end{proposition}
\begin{proof}
The details of the proof are deferred to \rev{Appendix A.2}. The main technical ingredient is based on Jensen's inequality and the concavity of penalty DB with respect 
to simplex variables $\mathbf{p}$.
\end{proof}
Inequality (\ref{imbalanced-bias-Dice}) means that the region-size term in Dice in Eq. (\ref{eq:logdice-multiclass}) is a penalty function for constraint $\mathbf{p}= \mathbf{t}$, where $\mathbf{t}$
is the simplex vertex given in Proposition \ref{prop:multi-class-Dice-Bias}. 
\rev{
Thus, this proposition demonstrates that the hidden bias term of Dice encourages the prediction result to be close to the simple vertex $\mathbf{t}=\left (\hat{t}{j} \right ){1 \leq j \leq K} \in {0, 1}^K$, where a specific region (i.e., the largest region according to the ground-truth labels and typically the background) includes all the pixels. Consequently, it can lead to imbalanced segmentation outputs, pushing the areas of the remaining classes (usually foreground) towards zero.
}
Therefore, it encourages extremely imbalanced segmentation prediction, where a specific region (i.e. the largest region according to the ground-truth labels) includes all the pixels and the remaining 
regions are empty. 
All in all, the logarithmic Dice loss integrates a hidden region-size prior preferring extremely imbalanced segmentations, which is optimized jointly with a ground-truth matching term 
similar to CE. It is worth noting that, in the two-class (binary) segmentation case, Dice might be used for the foreground region only, as in the popular work in \citep{vnet2016}, for instance. 
Similarly to the multi-class case discussed above, a single Dice also decomposes into a ground-truth matching term and region-size penalty, with the latter encouraging extremely imbalanced 
binary segmentations. We provide more details for this case in Appendix B.

\subsection{On the link between the different variants of Dice}

The region-size analysis we discussed above is based on the standard logarithmic Dice loss. Here, we argue that both logarithmic and linear Dice are very closely related and, hence, 
the linear Dice also hides a class-imbalance bias. In fact, from a constrained-optimization perspective, the two losses could be viewed as different penalty functions for imposing 
constraints: $\text{Dice}_k=1 \, \forall k$. Both functions -$\log(x)$ and $(1-x)$ are monotonically decreasing in $[0, 1]$ and achieve their minimum in $[0, 1]$ at $x=1$. Furthermore, the logarithmic Dice is 
an upper bound on the linear one. This follows directly from: $-\log \left (t \right ) \ge 1-t \quad \forall t >0$. Of course, this does not mean that optimizing these two variants leads to exactly 
the same results. The differences in their results might be due to optimization (i.e., different gradient dynamics stemming from logarithmic and linear penalties).

\subsection{The hidden region-size bias of CE}

In the following, we give an information-theoretic perspective of CE, via a generative view of network predictions and a Monte-Carlo approximation of the entropy of the learned 
features given the labels. This highlights a hidden region-size bias of CE, which encourages the proportions of the predicted segmentation regions to match the 
ground-truth proportions. \\ 

\begin{proposition}
\label{prop:ce}
Let ${\cal F}$ and ${\cal K}$ denote the random variables associated with the learned features and the labels, respectively, and ${\cal H}({\cal F}|{\cal K})$ the conditional entropy of learned features given the labels, estimated via Monte-Carlo :
\begin{equation}
\label{eq:monte-carlo-entropy-features}
{\cal H}({\cal F}|{\cal K}) \approx \sum_k^K y_k {\cal H}({\cal F}|{\cal K} = k) \approx - \frac{1}{|\mathbf{\Omega}|} \sum_k^K \sum_{i \in \mathbf{\Omega}_k} \log(\mathbb{P}(\mathbf{f}_i^{\theta}|k)) 
\end{equation}
where ${\cal H}({\cal F}|{\cal K} = k)$ is the empirical estimate of the conditional entropy of features 
given a specific class $k$ (expression in Table \ref{table:notations}) and $\mathbb{P}(\mathbf{f}_i^{\theta}|k)$ denotes the probability of the learned features given class $k$.    
We have the following generative view of {\em CE}:
\begin{align}
\label{eq:gce}
    \text{\em CE} \ceq \underbrace{\cal{H}(\cal{F}|{\cal{K}})}_{\text{Ground-truth matching}} + \underbrace{{\cal D}_\text{\em KL}({\mathbf y} || \mathbf{p})}_{\text{region-size bias}}
\end{align}
\end{proposition}

The detailed proof is deferred to Appendix A.3.
The approximation of ${\cal H}({\cal F}|{\cal K} = k)$ in the second line of Eq. (\ref{eq:monte-carlo-entropy-features})   
is based on the well-known Monte-Carlo estimation \citep{Kearns1998,TangMAB19Kernel}. Then the relationship in Eq. (\ref{eq:gce}) follows from Eq. (\ref{eq:monte-carlo-entropy-features}), after some manipulations, using 
Bayes rule $\mathbb{P}(\mathbf{f}_i^{\theta}|k) \propto \frac{p_{ik}}{\hat{p}_k}$ and $\sum_{i \in \mathbf{\Omega}_k} \log(\hat{p}_k) = |\mathbf{\Omega}_k| \log(\hat{p}_k)$.

This information-theoretic view of CE shows that the latter has an implicit (hidden) region-size bias towards the ground-truth region proportions (the KL term). 
This bias competes with the entropy term, which encourages low uncertainty (variations) within each ground-truth segmentation region $\mathbf{\Omega}_k$. 
The entropy term could be viewed as a ground-truth matching term: it reaches its global minima when the feature embedding is constant within each region.
If used alone, the entropy term may lead to trivial imbalanced solutions. The region-size KL term avoids such trivial solutions by matching the ground-truth 
class proportions. Note that there is no mechanism in CE to control the relative contributions of those two competing terms as they are implicit in CE.

\subsection{Our solution}
\label{sec:solution}

\begin{figure}[htb]
    \centering
    \includegraphics[width=0.8\columnwidth]{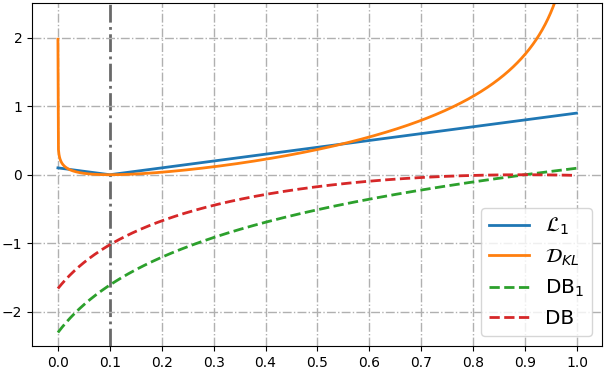}
    \caption{\textbf{Different region-size penalties.} The ground-truth foreground region proportion is set to $0.1$.
    The expression of penalty $\text{DB}_1$ is provided in Appendix B, and corresponds to the two-class (binary) variant of Dice, where the loss is used over the foreground region only.
    Penalty ${\cal L}_1$ presents better gradient dynamics at the vicinity of region size prediction $\hat{p}_1=0$.
    Best seen in color.}
    \label{fig:bias}
    \vspace{-5mm}
\end{figure}

Our analysis shows that Dice, CE and their combinations, e.g., $\text{CE} - \log(\text{Dice})$, are closely related and enforce two types of competing constraints : 
ground-truth matching and region-size constraints. However, there is no clear consensus in the literature as to which loss is better, with the 
performances of each varying across data sets and applications. This variability in performances could be explained by two fundamental factors:

$\bullet~$ {\bf The difference in the region-size prior}. The region-size priors are different as Dice has an intrinsic bias preferring very small regions, while CE encourages 
the right (ground-truth) region proportions. This might explain the wide experimental evidence in the medical imaging literature, where using or adding Dice losses brings improvements
for imbalanced segmentation with extremely small regions.

$\bullet~$ {\bf Weighting the contribution of the bias term}. Our analysis suggests that CE should be preferred over Dice in all cases and 
applications (both balanced/imbalanced segmentation, or segmentation problems with high variability in region proportions) as it promotes the right region-size distribution.
While this seems to be widely the case in natural image segmentation, where Dice is uncommon, the extensive experimental evidence in the medical-image segmentation literature 
suggests otherwise, especially in extremely imbalanced problems. We argue that this is due to the relative contribution of the region-size term in the overall objective. 
Controlling such region-size contribution is very important in imbalanced problems. In particular, it mitigates the difficulty that the ground-truth matching terms differ by several orders of 
magnitude across regions, as in CE, which causes large-region terms to completely dominate small-region ones. This analysis also resonates with the fact that combo losses such 
as CE - $\lambda \log$ (Dice) perform very competitively in imbalanced segmentation, as shown by \citep{WongMTS18logdice,taghanaki2019combo}, among several other recent works. In this case, 
controlling the relative contribution of each of these terms indirectly controls the weight of the region-size bias. Note that such control is not 
possible when using CE alone or Dice alone, as the region-size biases in these losses are hidden (implicit).

We propose a principled and simple solution, which enables to control explicitly the region-size bias, via regularization losses that encourage the correct class proportions 
and are used in conjunction with CE :
\begin{equation}
\label{eq:framework}
\text{RCE} = \text{CE} + \lambda {\cal R}(\mathbf{y}; \mathbf{p})
\end{equation}
Our region-size regularizers increase the contribution of the minority classes in imbalanced problems, but, unlike Dice, do not lose adaptability to problems 
with various class proportions. Our extensive experiments and ablation studies over different losses and applications demonstrate the effectiveness of 
our explicit region-size regularizers. We investigate different forms of regularization, including the ${\cal L}_1$ norm, i.e., 
${\cal R}(\mathbf{y}; \mathbf{p}) = {\cal L}_1 (\mathbf{y}, \mathbf{p})$, and
the KL divergence, i.e., 
${\cal R}(\mathbf{y}; \mathbf{p}) = {\cal D}_\text{KL}(\mathbf{{y}}||\mathbf{p})$; see Table \ref{table:notations} for the expressions of ${\cal D}_\text{KL}$ and ${\cal L}_1$.
In Fig. \ref{fig:bias}, we depict our different regularizers as functions of the region-size distribution for a binary-segmentation case, with the foreground-region proportion set to $0.1$, along with the bias terms in Dice.
While our ${\cal D}_\text{KL}$ and ${\cal L}_1$ regularizers may deliver comparable performances (see the experimental section), ${\cal L}_1$ might be a better option for extremely imbalanced segmentations, due to 
its gradient properties and stability at the vicinity of $0$, i.e., when the region-size probability $\hat{p}_1$ is close to $0$. Notice that, at the vicinity of zero, both first and second derivatives of the 
regularizer are {\em unbounded} for ${\cal D}_\text{KL}$, but {\em bounded} and {\em constant} for ${\cal L}_1$. Our experiments on imbalanced medical image segmentation confirm the effectiveness of the ${\cal L}_1$ regularizer.

\section{Experiments}

\begin{table}[t]
\begin{center}
\caption{\textbf{Quantitative evaluations of different losses on Kvasir and CVC-ClinicDB test sets. } All the models were conducted over three independent runs under the optimal hyper-parameters, and we report their average scores with standard deviations achieved on the test sets. Best method is highlighted in bold, whereas the second best method is underlined. 
}
\label{tab:polyp}
\resizebox{1.0\textwidth}{!}
{
\begin{tabular}{@{}lccccccccccc@{}}
\toprule
& \multicolumn{5}{c}{ KVasir} && \multicolumn{5}{c}{CVC-ClinicDB} \\
\multirow{3}{*}{Loss} & \multicolumn{2}{c}{\footnotesize R50FPN} && \multicolumn{2}{c}{\footnotesize{R50UNet}} && \multicolumn{2}{c}{\footnotesize{R50FPN}} && \multicolumn{2}{c}{\footnotesize{R50UNet}} \\
\cline{2-3} \cline{5-6} \cline{8-9} \cline{11-12}
 & \footnotesize{DSC(\%)} & \footnotesize{IoU(\%)}  && \footnotesize DSC(\%) & \footnotesize IoU(\%) && \footnotesize DSC(\%) & \footnotesize IoU(\%)  && \footnotesize DSC(\%)  & \footnotesize IoU(\%)  \\
\midrule
CE & $89.3\pm0.7$ & $80.7\pm1.1$  && $88.7\pm0.6$ & $79.7\pm0.9$ && $91.0\pm0.3$ & $84.2\pm0.6$ && $\underline{92.9\pm0.3}$ & $\bf 86.7\pm0.5$ \\ 
\rev{WCE} & \rev{$88.6 \pm 0.7$} & \rev{$79.5 \pm 1.1$} && \rev{$89.2 \pm 0.5$} & \rev{$80.5 \pm 0.9$} && \rev{$91.0 \pm 1.0$} & \rev{$83.5 \pm 1.8$} && \rev{$91.6 \pm 0.8$} & \rev{$84.6 \pm 1.3$}\\
FL & $89.1\pm0.6$ & $80.2\pm1.0$ && $89.5\pm0.7$ & $80.9\pm1.1$ && $91.2\pm0.9$ & $83.8\pm1.4$ && $92.2\pm0.8$ & $85.7\pm1.3$ \\
Dice & $89.5\pm0.4$ & $81.1\pm0.8$ && $89.1\pm0.3$ & $80.3\pm0.5$ && $91.8\pm0.8$ & $85.2\pm1.5$ && $92.3\pm0.5$ & $85.8\pm0.8$ \\
LogDice & $89.4\pm0.4$ & $80.9\pm0.7$ && $88.7\pm1.1$ & $79.8\pm1.8$ && $91.5\pm0.7$ & $84.3\pm1.1$ && $92.5\pm0.5$ & $86.1\pm0.9$ \\
\midrule
DiceCE & $89.4\pm0.2$ & $80.8\pm0.3$ && $89.4\pm0.9$ & $80.8\pm1.4$ && $90.5\pm0.6$ & $82.6\pm1.1$ && $91.9\pm0.7$ & $85.0\pm1.1$  \\
DiceFL & $89.7\pm0.4$ & $81.2\pm0.7$ && $89.6\pm0.9$ & $81.2\pm1.4$ && $91.2\pm0.6$ & $83.8\pm1.1$ && $91.5\pm0.9$ & $84.4\pm1.5$ \\
LogDiceCE & $\underline{90.3\pm0.6}$ & $\underline{82.3\pm0.9}$ && $89.3\pm0.6$ & $80.7\pm0.9$ && $91.5\pm0.7$ & $84.3\pm1.1$ && $92.8\pm0.7$ & $\underline{86.6\pm1.1}$ \\
LogDiceFL & $89.4\pm0.3$ & $80.8\pm0.4$ && $89.6\pm0.9$ & $81.2\pm1.4$ && $90.9\pm0.4$ & $83.3\pm0.7$ && $92.3\pm0.4$ &  $85.7\pm0.8$ \\
DBCE & $89.8\pm0.4$ & $81.4\pm0.6$ && $\underline{90.0\pm0.4}$ & $\underline{81.8\pm0.8}$ && $91.3\pm0.7$ & $84.0\pm1.1$ && $91.8\pm0.6$ & $84.8\pm1.0$ \\
\midrule
RFL(${\cal D}_\text{KL}$)  & $89.4\pm0.3$ & $80.8\pm0.5$ && $89.2\pm0.6$ & $80.6\pm1.0$ && $90.4\pm0.4$ & $82.4\pm0.7$ && $92.7\pm0.2$ & $86.2\pm0.3$ \\
RFL(${\cal L}_1$) & $89.8\pm0.1$ & $81.6\pm0.2$ && $89.9\pm0.2$ & $\underline{81.8\pm0.3}$ && $\underline{92.5\pm0.5}$ & $\underline{86.1\pm0.9}$ && $92.8\pm0.3$ & $86.5\pm0.6$  \\
RCE(${\cal D}_\text{KL}$)  & $89.3\pm0.3$ & $80.7\pm0.4$ && ${89.9\pm0.4}$ & $81.7\pm0.7$ && $90.2\pm0.6$ & $82.2\pm0.9$ && $92.7\pm0.4$ & $86.4\pm0.6$ \\
{\bf RCE(${\cal L}_1$)} & $\bf 90.7\pm0.4$ & $\bf 83.0\pm0.6$ && $\bf 90.3\pm0.4$ & $\bf 82.2\pm0.6$ && $\bf 93.5\pm0.1$ & $\bf 87.7\pm0.1$ && $\bf 93.2\pm0.8$ & $\bf 86.7\pm1.1$  \\
\bottomrule
\end{tabular}
}
\end{center}
\vspace{-5mm}
\end{table}

\subsection{Experimental settings}

\noindent \textbf{Datasets.} We first evaluate all the losses on two 2D medical image segmentation applications, including two Polyp segmentation benchmarks, i.e. KVasir and CVC-ClinicDB, and one Retinal Lesions segmentation dataset \citep{Wei2020RetinalLesion}.
Then, we extend all the losses to 3D medical image segmentation and evaluate on \revA{three} standard datasets, i.e.  \revA{Pancreas \& Tumor,  Liver \& Tumor \citep{antonelli2021medical}, and AMOS\citep{ji2022amos}}. Here, we present the description of all the data sets used in our experiments.

$\bullet~$ \textbf{KVasir \citep{jha2020kvasir} and CVC-ClinicDB \citep{BERNAL201599cvc}} are two popular polyp datasets. KVasir contains $1,000$ samples in highly variant resolutions from $487\times332$ to $1920\times1072$ pixels, and CVC-ClinicDB contains $612$ images of $388\times284$ size.
Samples from both sets are collected from substantial specularities, with great variability in polyp types, sizes and appearances.
For both data sets, we follow the training setting in \citep{fan2020pra} by using $80\%$ samples for training, $10\%$ for validation, and by evaluating on the rest as testing samples.

$\bullet~$ \textbf{Retinal Lesions} \citep{Wei2020RetinalLesion} is a large collection of color fundus images.
A panel of 45 experienced ophthalmologist was formed to label this dataset and each image was assigned to at least three annotators to get trustworthy pixel-level lesions annotations \citep{Wei2020RetinalLesion}.
In our experiments, we employ its public version\footnote{\url{https://github.com/WeiQijie/retinal-lesions}}, consisting of $1,593$ samples and conduct our experiments in the binary scenario (i.e. segmenting the lesion region versus background).
The data set is randomly divided into training ($70\%$), validation ($10\%$) and testing ($20\%$) sets, with the images being resized to $512\times 512$.

$\bullet~$ \textbf{\revA{Pancreas \& Tumor}} and \textbf{\revA{Liver \& Tumor}} data sets are provided in the Medical Segmentation Decathlon \citep{antonelli2021medical}.
\revA{Pancreas \& Tumor} includes 281 portal-venous phase 3D CT scans of patients undergoing resection of pancreatic masses. The segmenting targets consists of two categories, i.e., pancreas and tumor.
For \revA{Liver \& Tumor}, it consists of \rev{131 contrast-enhanced CT cases with public available labels}.
The corresponding regions of interest are the segmentation of the liver and tumors inside the liver.
On both datasets, the unbalance between the large (background), medium (pancreas or liver) and small (tumor) structures makes them significantly challenging. 
The dataset is randomly split into $80\%$ for training and $20\%$ for testing.

\rev{
$\bullet~$ \textbf{AMOS} \citep{ji2022amos}, i.e., the Abdominal Multi Organ Segmentation 2022 challenge, is a large-scaled benchmark for  abdominal multi-organ segmentation from CT/MRI scans.
It consists of 600 CT/MRI scans, where 15 categories of organs are labeled.
The samples are splitted into $200+40$ (CT+MRI) for training, $100+20$ for validation and $200+40$ for testing.
In our experiments, we utilize Task 1 in the challenge, i.e., multi-organ segmentation on the CT Images, to evaluate the performances of our loss in comparison with baselines.
We train all the models on the training set, and report performances on the validation set.
}

\noindent \textbf{Baselines.}
We compare the proposed loss function in Eq. (\ref{eq:framework}) under two different penalty terms, i.e., RCE(${\cal D}_\text{KL}$) and RCE(${\cal L}_1$), with widely used losses like the cross entropy (CE), focal loss (FL), standard Dice loss (Dice), and logarithmic Dice loss (LogDice), as well as with competitive compound losses like DiceCE, DiceFL, LogDiceCE and LogDiceFL \citep{ma2021loss}.
\rev{
Note that all the compound losses consist of a balancing weight $\lambda$ to control the relative contribution of the two terms.
For hyper-parameters (i.e., the balancing weight $\lambda$) tuning, we perform ablation study on the 2D segmentation tasks, as shown in Fig.~\ref{fig:lambda}.
We find the the best setting for each loss is consistent on both tasks, then we fix the balancing weights on the 3D segmentation applications.
In particular, the $lambda$ for the proposed loss with ${\cal L}_1$ regularizer ($RCE({\cal L}_1)$ and $RCE({\cal D}_\text{KL}$) is set to $1.0$, while it is set to $0.1$ when we use ${\cal D}_\text{KL}$.
The balancing weights for Dice related losses, including DiceCE, DiceFL, LogDiceCE and LogDiceFL, are all set to $0.1$.
}
\rev{Note that for Focal loss, we keep the setting from the original paper \citep{Lin2017FocalLoss}, i.e., $\gamma = 2$.}
In addition, we also evaluate the performance of incorporating focal loss with region-size penalties, denoted as as RFL(${\cal D}_\text{KL}$) and RFL(${\cal L}_1$).
In our implementation for computing the predicted region proportion, we modify the soft-max function with a temperature parameter. This enables a better estimate of the actual region proportion (refer to Appendix C for details).

\noindent \textbf{Training details.}
On the 2D medical image segmentation benchmarks, i.e. KVasir \citep{jha2020kvasir}, CVC-ClinicDB \citep{BERNAL201599cvc} and Retinal Lesions \citep{Wei2020RetinalLesion}, the standard encoder-decoder segmentation network is used with variant options of encoder and decoder structures (R50FPN and R50UNet), whose implementations are publicly available\footnote{\url{https://github.com/qubvel/segmentation_models.pytorch}}.
We train the model during $60$ epochs with batch size set to $8$ via Adam optimizer.
The initial learning rate is set to 1e-4 and halved if the loss on the validation set does not decrease within 5 epochs.
Regarding the 3D medical segmentation benchmarks of Pancreas and Liver, \rev{and AMOS}, we employ the state-of-the-art 3D nnUNet \citep{Fabian2021nnUnet}.
The SGD optimizer is used for $1$k epochs with a batch size of $2$.
The initial learning rate is set as $0.01$ and decayed throughout the training via polynomial strategy to linearly scale down to $0$.

\noindent \textbf{Evaluation metrics.}
We report standard metrics on each dataset following previous works.
On the polyp applications of KVasir and CVC-ClinicDB, we report both Dice Similarity Coefficient (DSC) and Intersection over union (IoU) following \citep{fan2020pra}.
For Retinal Lesions, besides DSC, we also include a boundary-based metric, \rev{Normalised Surface Distance (NSD) \citep{Stanislav2018clinic}}.
Finally, on the 3D medical imaging, we use the standard DSC as the evaluation measure.

\subsection{Results}

\begin{table}[htb]
\begin{center}
\caption{\textbf{Quantitative evaluations of different losses on Retinal Lesions.} Average DSC and \rev{NSD} values (and standard deviation over three independent runs) achieved on the test set are reported. Note that $\text{Dice}_1$ is implemented for all the Dice related losses for the binary setting on this dataset, and DiceBias here refers to region-size bias for the binary Dice (details can be found in Appendix B). 
}
\label{tab:retinal}
{
\begin{tabular}{@{}lccccc@{}}\toprule
\multirow{3}{*}{Loss} & \multicolumn{2}{c}{\footnotesize{R50FPN}} && \multicolumn{2}{c}{\footnotesize{R50UNet}} \\
\cmidrule{2-3} \cmidrule{5-6}
& {\footnotesize DSC (\%)} & \rev{{\footnotesize NSD (\%)}} && {\footnotesize DSC (\%)} & \rev{{\footnotesize NSD (\%)}} \\
\midrule
CE & $52.7\pm0.1$ & \rev{$14.6\pm0.7$} && $52.7\pm0.3$ & \rev{$15.9\pm0.5$} \\ 
\rev{WCE} & \rev{$53.3\pm0.4$} & \rev{$14.7\pm0.2$} &&  \rev{$53.4 \pm 0.3$} & \rev{$15.8 \pm 0.5$} \\ 
FL & $52.9\pm0.4$ & \rev{$15.2\pm0.4$} && $51.8\pm0.6$ & \rev{$16.0\pm0.6$} \\
Dice & $52.0\pm0.7$ & \rev{$14.9\pm0.7$} && $53.2\pm0.1$ & \rev{$15.6\pm0.3$} \\
LogDice & $52.0\pm1.0$ & \rev{$14.8\pm0.3$} && $53.5\pm0.5$ & \rev{$15.5\pm0.5$} \\
\midrule
DiceCE & $53.2\pm0.5$  & \rev{$14.7\pm0.6$} && $53.0\pm0.9$ & \rev{$16.2\pm0.3$} \\
DiceFL & $52.6\pm0.5$ & \rev{$15.2\pm0.3$} && $53.4\pm0.6$ & \rev{$16.3\pm0.4$} \\
LogDiceCE & $53.2\pm0.5$ & \rev{$15.3\pm0.3$} && $53.4\pm0.3$ & \rev{$16.1\pm0.2$} \\
LogDiceFL & $53.6\pm0.8$ & \rev{$15.1\pm0.8$} && $53.6\pm1.0$ & \rev{$15.5\pm0.9$} \\
DBCE &  $53.4\pm0.4$ & \rev{$15.5\pm0.5$} && $53.2\pm0.1$ & \rev{$16.7\pm0.3$} \\
\midrule
RFL & $53.8\pm0.4$ & \rev{$15.8\pm0.1$} && $53.7\pm0.6$ & \rev{$17.0\pm0.2$} \\
{\bf RCE} &  $\bf 54.5\pm0.2$ & \rev{$\bf 16.0\pm0.3$} && $\bf 54.3\pm0.2$ &  \rev{$\bf 17.5\pm0.4$} \\
\bottomrule
\end{tabular}
}
\end{center}
\end{table}

\begin{figure}[htb]
\centering
\includegraphics[width=1.0\textwidth]{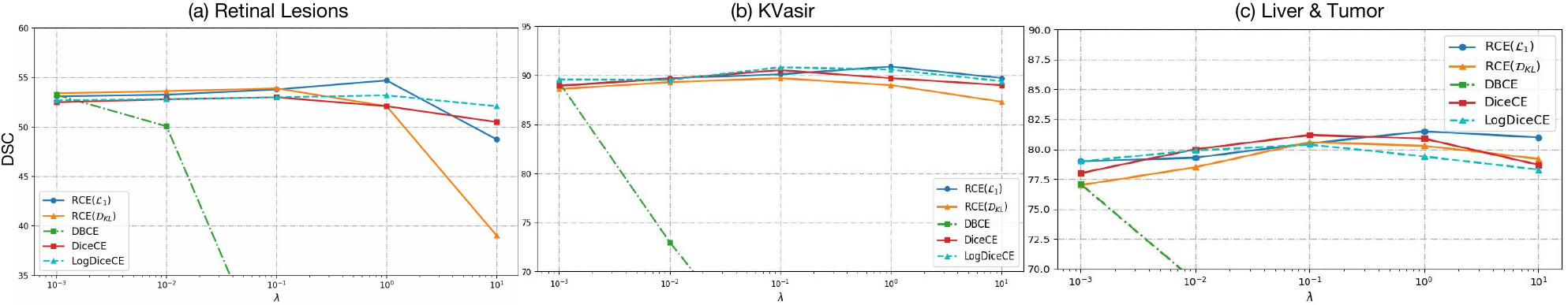}
\caption{\rev{\textbf{Ablation study on the balancing weight $\lambda$.} The performances of different compounding losses on the test set of (a) Retinal Lesions, (b) KVasir \revA{and (c) Liver \& Tumor} with different values of\ 
the balancing weight $\lambda$. For the 2D image datasets, the network here is fixed to R50FPN.}}
\label{fig:lambda}
\end{figure}

\subsubsection{Results on polyp datasets.}

The results on the two polyp datasets, i.e. KVasir and CVC-ClinicDB, are reported in Table~\ref{tab:polyp}.
To reduce the effect of randomness in the experiments, we report the average scores with the standard deviations over three independent runs for each model.
We present all the results into three groups due to the types of losses: single-term losses, compound losses and our proposed losses.
As seen in Table~\ref{tab:polyp}, the best of our proposed losses, RCE(${\cal L}_1$), consistently achieve the best performance over all the networks and metrics.
On R50FPN and KVaisir data set, for example, and in comparison to the CE baseline, the best performing among our methods integrating CE and the region-size bias brings $1.4\%$ and $2.3\%$ absolute improvements in terms of DSC and IoU, respectively. On the CVC-ClinicDB data set, and with the same model, the absolute improvements are $2.5\%$ in DSC and $3.5\%$ in IoU.
Moreover, our simple compound loss consistently outperforms composite losses integrating CE and Dice, as well as those integrating FL and Dice.
Theses results empirically validate our theoretical perspective, and shows that better and more stable results could be achieved with simple, explicitly controllable region-size terms.
Another interesting finding is that DBCE, the compound loss of CE and DB (DB denotes the bias term of dice loss in Eq.~\ref{eq:logdice-multiclass}), is able to yield similar scores to DiceCE and LogDiceCE, further validating the observation that the main difference between CE and Dice is in their region-size terms.

Regarding our method in Eq.~\ref{eq:framework}, different forms of the region-size penalties could be used.
We investigate two discrepancy measures on region proportions, i.e., ${\cal L}_1$ and ${\cal D}_\text{KL}$.
As shown in Table~\ref{tab:polyp}, RCE(${\cal L}_1$) and RFL(${\cal L}_1$) delivers better performances than RCE(${\cal D}_\text{KL}$) and RFL(${\cal D}_\text{KL}$).
For example, the R50FPN model trained with RCE(${\cal L}_1$) achieves an average DSC of $90.7\%$ on the KVasir test set, which corresponds to an improvement of $1.4\%$ over RCE(${\cal D}_\text{KL}$).
It yields $93.5\%$ under the same setting on CVC-ClinicDB, bringing over $3\%$ improvement over RCE(${\cal D}_\text{KL}$).
This is mainly due to the more stable gradient dynamics of ${\cal L}_1$, as shown in Fig.~\ref{fig:bias}.
Thus, we will only report the performances with ${\cal L}_1$ penalty for the rest of our experiments, denoted as RCE and RFL.
We note that the same trend remains when we replace CE with FL in Eq.~\ref{eq:framework}.

\begin{figure*}[htb]
    \centering
    \includegraphics[width=1.0\textwidth]{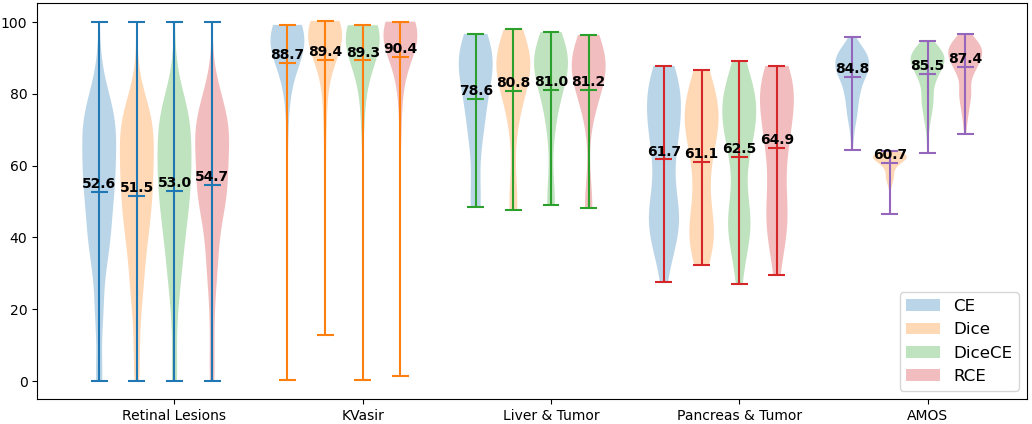}
    \caption{\rev{Violin plots showing the DSC distribution on \revA{the five datasets} for different methods. The network is fixed to R50FPN.}}
    \label{fig:violin}
\end{figure*}

\begin{figure*}[htb]
    \centering
    \includegraphics[width=1.0\textwidth]{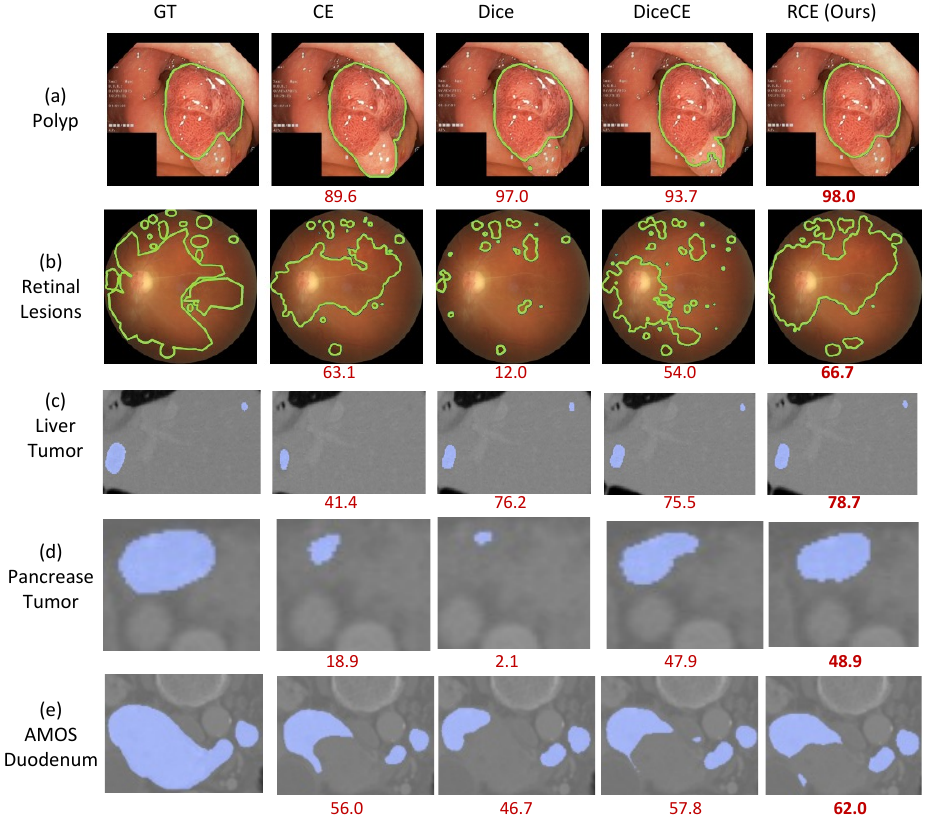}
    \vspace{-2mm}
    \caption{\textbf{Visual results of different segmentation losses.} Examples are from the four datasets : (a) Polyp, (b) Retinal Lesions, (c) Liver Tumor, (d) Pancreas Tumor \revA{and (e) AMOS Duodenum}. The ground-truth is provided in the first column. At the bottom of each prediction, we indicate the corresponding DSC ($\%$) score.}
    \label{fig:example}
    \vspace{-2mm}
\end{figure*}

\subsubsection{Results on Retinal Lesions.}

Table \ref{tab:retinal} reports the quantitative comparison between the proposed loss and the related methods on Retinal Lesions benchmark.
Regardless of the networks and metrics, our method (RCE) consistently achieves the best performance over different settings.
On R50FPN, for instance, compared to CE, our best model, i.e., RCE with ${\cal L}_1$  as the region-size regularizer, brings nearly $2.0\%$ improvement in terms of DSC and \rev{$1.4\%$ in terms of NSD}.
\rev{While DSC and NSD highlight different dimension of the result, i.e., internal filling of the target region and segmentation boundary, the proposed model is demonstrated to be more effective for both metrics.}
Regarding the robustness, the proposed losses present a more stable performance across different runs and backbones, which is reflected in the relatively lower variances. This can be explained by their better adaptability to different target sizes and better gradient dynamics at the vicinity of predicted region size $\hat{p}_k=0$ (as shown in Fig. \ref{fig:bias}).
Similar to the results on the polyp datasets, one could make the same observation: The scores of combining CE with the bias term of Dice (DBCE) are close to the CE-Dice combo losses (i.e. DiceCE and LogDiceCE). This further confirms our theoretical insight stating that the fundamental difference between CE and Dice lies in their distinct hidden region-size biases.
\rev{In Fig.~\ref{fig:violin}, we show the violin plots regarding the DSC distribution on Retinal lesions for different methods, which presents similar conclusion as the quantitative scores that our simple method is able to improve the prediction distribution across all the samples.}

\subsubsection{Ablation study on the balancing weight.}

We study the impact of the balancing weight $\lambda$ in the proposed loss in Eq. (\ref{eq:framework}), as well as the balancing weight in the other composite losses like DiceCE and LogDiceCE, presented in Fig. \ref{fig:lambda}.
It is empirically found that the best $\lambda$ values for different penalty terms are consistent on \revA{three different datasets, including both 2D and 3D images}: $1.0$ for RCE(${\cal L}_1$), $0.1$ for RCE(${\cal D}_\text{KL}$), and $0.1$ for DiceCE and LogDiceCE.
From the curves, we can notice that ${\cal L}_1$ is a better choice than KL divergence for the proposed loss because of its better stability.
This may relate to its gradient properties and stability at the vicinity of $0$ as shown in Fig. \ref{fig:bias}.
Thus, we can use a relatively larger weighting value in RCE(${\cal L}_1$).
For the loss integrating CE with the Dice bias term, i.e., DBCE, we demonstrate that it can yield performances similar to DiceCE and LogDiceCE, but it drops  significantly with high weighting values ($\lambda \geq 0.1$).
This might be due to its gradient characteristics. 
Comparing to the widely suggested composite loss of CE and Dice, our method deliver better performance with the same hyper-parameter budget.
\rev{Note, we use the best empirical values of $\lambda$ presented in Fig.~\ref{fig:lambda} for the experiments on 3D medical image segmentation.}

\begin{table}[htb]
\begin{center}
\caption{\rev{\textbf{Results on two 3D medical image segmentation, i.e., \revA{Pancreas \& Tumor, and Liver \& Tumor}} We report the average DSC scores with standard deviations on validation set, across three independent runs.}
}
\vspace{-2mm}
\label{tab:3dmed}
\resizebox{0.9\columnwidth}{!}
{
\begin{tabular}{@{}llcccccccccc@{}}
\toprule
&& \multicolumn{3}{c}{\footnotesize{\revA{Pancreas \& Tumor}}} && \multicolumn{3}{c}{\footnotesize{\revA{Liver \& Tumor}}}  \\
\cmidrule{3-5} \cmidrule{7-9}
Loss && \revA{pancreas} & \revA{tumor} & mean && \revA{liver} & \revA{tumor} & mean  \\
\midrule
CE && \rev{$81.9 \pm 0.5$} & \rev{$41.7 \pm 0.8$} & \rev{$61.8 \pm 0.7$} && \rev{96.5 $\pm$ 0.2} & \rev{61.1 $\pm$ 1.0} & \rev{78.8 $\pm$ 0.4} \\ 
FL 
&& \rev{$81.6 \pm 0.4$} & \rev{$42.4 \pm 0.5$} & \rev{$62.0 \pm 0.4$} && \rev{96.6 $\pm$ 0.2} & \rev{63.4 $\pm$ 0.6} & \rev{80.0 $\pm$ 0.4} \\
Dice
&& \rev{$80.8 \pm 0.8$} & \rev{$43.3 \pm 1.0$} & \rev{$61.9 \pm 0.8$} && \rev{96.4 $\pm$ 0.2} & \rev{65.1 $\pm$ 0.5} & \rev{80.8 $\pm$ 0.3} \\
LogDice 
&& \rev{$81.8 \pm 0.7$} & \rev{$42.7 \pm 1.0$} & \rev{$62.3 \pm 0.8$} && \rev{96.4 $\pm$ 0.3} & \rev{63.3 $\pm$ 0.6} & \rev{79.8 $\pm$ 0.4} \\
\midrule
DiceCE  && \rev{$81.7 \pm 0.4$} & \rev{$43.1 \pm 0.4$} & \rev{$62.4 \pm 0.4$} && \rev{96.5 $\pm$ 0.2} & \rev{65.9 $\pm$ 0.4} & \rev{81.2 $\pm$ 0.3} \\
DiceFL && \rev{$81.9 \pm 0.4$} & \rev{$43.6 \pm 0.3$} & \rev{$62.8 \pm 0.4$}  && \rev{96.6 $\pm$ 0.2} & \rev{64.9 $\pm$ 0.5} & \rev{80.8 $\pm$ 0.3} \\
LogDiceCE 
&& \rev{$81.5 \pm 0.2$} & \rev{$43.1 \pm 0.2$} & \rev{$62.3 \pm 0.2$} && \rev{\bf 96.7 $\pm$ 0.1} & \rev{64.1 $\pm$ 0.5} & \rev{80.4 $\pm$ 0.3} \\
LogDiceFL && \rev{$82.0 \pm 0.3$} & \rev{$45.5 \pm 0.2$} & \rev{$63.8 \pm 0.2$} && \rev{96.5 $\pm$ 0.4} & \rev{65.5 $\pm$ 0.5} & \rev{81.0 $\pm$ 0.4} \\
\midrule
RFL && \rev{$82.1 \pm 0.2$} & \rev{$45.5 \pm 0.5$} & \rev{$63.8 \pm 0.3$} && \rev{96.4 $\pm$ 0.3} & \rev{64.8 $\pm$ 0.5} & \rev{80.6 $\pm$ 0.3} \\
{\bf RCE} && \rev{$\bf 82.3 \pm 0.2$} & \rev{$\bf 47.7 \pm 0.3$} & \rev{$\bf 65.0 \pm 0.2$} && \rev{96.5 $\pm$ 0.3} & \rev{\bf 66.4 $\pm$ 0.4} & \rev{\bf 81.5 $\pm$ 0.3} \\
\bottomrule
\end{tabular}
}
\end{center}
\end{table}

\subsubsection{Results on 3D medical image segmentation.}
We now investigate the performance on two 3D medical image segmentation benchmarks, i.e., Pancreas and Liver, whose results are reported in Table~\ref{tab:3dmed}.
The proposed approach yielded the best performance in most cases, which is consistent with the empirical results on 2D images shown in Table~\ref{tab:polyp} and Table~\ref{tab:retinal}.
In particular, the improvement on the more challenging class (small structured tumor) are significant. Compared to CE, for instance, our best model increases the DSC of tumor on Pancreas and Liver by $4.6\%$ and $6.6\%$, respectively.
This further validates our theoretical analysis, which suggests that Dice-related losses (linear and logarithmic Dice) can bring improvements over CE for the categories with small region sizes, while the proposed method is more robust to variations in region sizes, yielding a better overall segmentation performances.

\begin{table}[htb]
\begin{center}
\caption{\textbf{Results on AMOS dataset.}
\revA{Average DSC and NSD values (and standard deviation over three independent runs) achieved on the validation set are reported.
Mean DSC (mDSC) score on validation is reported for each method.}}
\label{tab:amos}
\begin{tabular}{@{}lcc@{}}
\toprule
\rev{Loss} & \rev{mDSC (\%)} & \revA{NSD (\%)} \\
\midrule
\rev{CE} & \rev{84.7 $\pm$ 0.3} & \revA{72.5 $\pm$ 0.2} \\
\rev{FL} & \rev{83.8 $\pm$ 0.5} & \revA{70.1 $\pm$ 0.5} \\
\rev{Dice} & \rev{61.0 $\pm$ 0.7} & \revA{54.1 $\pm$ 0.8} \\
\rev{LogDice} & \rev{42.3 $\pm$ 1.0} & \revA{36.0 $\pm$ 1.3} \\
\rev{DiceCE} & \rev{84.6 $\pm$ 0.5} & \revA{72.3 $\pm$ 0.6} \\
\rev{LogDiceCE} & \rev{86.9 $\pm$ 0.1} & \revA{75.5 $\pm$ 0.2} \\
\midrule
\rev{\textbf{RCE (Ours)}} & \rev{\textbf{87.5 $\pm$ 0.3}} & \revA{\textbf{76.5 $\pm$ 0.5}} \\
\midrule
\end{tabular}
\end{center}
\end{table}

\rev{
Table~\ref{tab:amos} presents the mean DSC score for different methods on AMOS benchmark.
Under the multiple categories setting, our method, RCE, also outperforms all the baselines, further demonstrating the robustness of the proposed loss.
It is worth noting that the mean scores of singular Dice and LogDice losses decrease significantly compared to the baseline of CE, which is due to the failure to handle some particular classes, as shown in the per-class performances reported in Appendix E.
This phenomenon consolidates our theoretical insights that Dice-related losses have limitations in adapting to different categories with diverse region areas.

}

\subsubsection{Qualitative results.}
\rev{Some positive visualized examples from medical datasets are presented in Fig. \ref{fig:example}, providing subjective insights on the benefits of our method}. 
We can observe that the model trained with Dice variants (\textit{third} and \textit{fourth column}) tends to under-segment large and medium target regions (Top, Bottom), while CE could miss some small structures (Middle). In contrast, the proposed solution enables a better trade-off between finding small regions, reducing the number of false positives and matching the size of the larger targets. %
Furthermore, one can notice that our loss yields better consistencies with the target region proportions than the others.

\section{Conclusion}
We provided a detailed theoretical analysis of the two most popular semantic segmentation losses, i.e., Cross-entropy and Dice, which revealed non-obvious bounding relationships 
and hidden region-size biases, suggesting that CE is a better option in general.
Then, we showed how both loss functions could be written under a common formulation, containing a ground-truth matching term and a region-size bias.
The implicit bias in Dice prefers small regions, improving its performance in highly imbalanced conditions, as in medical-imaging applications.
The bias hidden in CE encourages the ground-truth region proportion, which makes it a generally better option in complex scenarios with diverse class proportions.
Furthermore, we proposed a principled solution, which enables to control the region-size bias via ${\cal L}_1$ and KL penalties that encourage the target class proportions, while 
improving training stability. 
Our flexible formulation enables the minority classes to have better influence on training, without losing adaptability to medium-to-large regions. Extensive experiments on four benchmarks, covering various 2D and 3D medical-imaging applications, validate the theoretical analysis in this paper, as well as the effectiveness of the presented solution.
\rev{While the proposed method show promising results under different scenarios, it may have limitation in satisfying specific requirements. For example, certain applications might prioritize higher recall over precision, where more appropriate losses should be preferred.}
\revA{
Consequently, exploring loss functions that could adapt to the specific needs of individual domains or applications could be an interesting avenue for future research.
One possible direction is to investigate different constraints or penalties for different requirements.
Another important topic is to study the optimizer on its effectiveness to achieve the objective of the loss, especially when we train or fine-tune a large-scaled foundation model, like SAM \citep{kirillov2023segment}.
}

\section*{Acknowledgements}

This work is supported by the National Science and Engineering Research Council of Canada (NSERC), via its Discovery Grant program and FRQNT through the Research Support for New Academics program.
It is also supported by Prompt Quebec, and we thank Calcul Quebec and Compute Canada for providing computing resources.
Adrian Galdran was funded by a Marie Sklodowska-Curie Fellowship (No 892297).

\bibliographystyle{model2-names.bst}\biboptions{authoryear}
\bibliography{main_arxiv}

\clearpage
\appendix

\section{Proofs}

\rev{
\subsection{Proposition 1}
\label{appendix:prop1}

\begin{proof}
According to Jensen's Inequality, for a convex function $f$, we have:
\begin{equation}
\label{eq:jenson}
f(\frac{1}{n}\sum_{i=1}^{n} x_i) \leq \frac{1}{n}\sum_{i=1}^{n} f(x_i)
\end{equation}
where $x_i$ is a random variable, and $n$ is a positive integer.

Then given the convexity of function $-\log(x)$, we obtain:
\begin{align}
\label{eq:df-bound-ce}
DF &= - \sum_{k=1}^K \log\left (\frac{1}{|\mathbf{\Omega}_k|}\sum_{i\in \mathbf{\Omega}_k} p_{ik} \right )\nonumber\\
&\leq - \sum_{k=1}^K \frac{1}{|\mathbf{\Omega}_k|} \sum_{i\in \mathbf{\Omega}_k} \log(p_{ik})\nonumber\\
&= CE
\end{align}

\end{proof}

}

\subsection{Proposition 2}
\label{appendix:prop2}

\begin{proof}
Let us write the region-size bias term in the logarithmic Dice as a vector-valued function of probability simplex vector $\mathbf{p}$:
\begin{equation}
\label{region-size-bias-Dice-appendix}
g(\mathbf{p}) = \mathbf{1}_{K}^\top \log (\mathbf{p} + \mathbf{y}) 
\quad \mbox{for} \quad \mathbf{p} = \left (\hat{p}_{k} \right )_{1 \leq k \leq K} \in \Delta_K     
\end{equation}
where symbol $\top$ denotes transpose and $\mathbf{1}_{K}$ is the $K$-dimensional vector of ones. Function $g$ is concave because
its Hessian is a negative semi-definite matrix: 
The Hessian of $g$ is a diagonal matrix whose diagonal elements are given by $-\frac{1}{\hat{p}_{k}^2}$ and, hence, are all 
non-positive. Therefore, using Jensen's inequality and the fact that $\mathbf{p}$ is within the simplex, we have 
the following lower bound on penalty $g$ in Eq. (\ref{region-size-bias-Dice-appendix}):
\begin{equation}
\label{upper-bound-Dice-appendix}
g(\mathbf{p}) = g\left(\sum_{k=1}^K \hat{p}_{k} \mathbf{e}_k\right) 
\geq 
\sum_{k=1}^K \hat{p}_{k}  g(\mathbf{e}_k) 
\end{equation}
where $\mathbf{e}_k \in \{0, 1\}^K$  denote the $k$-th vertex of the simplex: the $k$-th component of  
$\mathbf{e}_k$ is equal to $1$ while the other components are all equal to $0$. 
Now, recall the definition of simplex vector $\mathbf{t}=\left (\hat{t}_{j} \right )_{1 \leq j \leq K}$: $\hat{t}_{j} = 1$ when $\hat{y}_{j} = \max_{\tiny{1 \leq k \leq K}} \hat{y}_{k}$ and $\hat{t}_{j} = 0$ otherwise. 
Given this definition, one could easily verify the following fact: 
\begin{equation}
\label{upper-bound-Dice-appendix-2}
g(\mathbf{e}_k) \geq g(\mathbf{t}) \quad \forall k 
\end{equation}
To see this, let $j$ denotes the integer verifying $\hat{y}_{j} = \max_{\tiny{1 \leq k \leq K}} \hat{y}_{k}$ and let $k \neq j$. Then, we have: 
\begin{equation}
g(\mathbf{e}_k) - g(\mathbf{t}) = \log\left (1+\frac{1}{\hat{y}_{k}} \right ) - \log\left (1+\frac{1}{\hat{y}_{j}} \right ) \geq 0. 
\end{equation}
This is due to the fact that function $\log \left (1+\frac{1}{x} \right )$ is monotonically decreasing in $[0, 1]$ and $\hat{y}_{k} \leq \hat{y}_{j}$. 
Now, combining inequalities (\ref{upper-bound-Dice-appendix}) and (\ref{upper-bound-Dice-appendix-2}), and using the fact that $\sum_{k=1}^K \hat{p}_{k}$=1, we obtain:
\begin{equation}
g(\mathbf{p}) \geq g(\mathbf{t}) 
\end{equation}
\end{proof}

\subsection{Proposition 3}
\label{appendix:prop3}

\begin{proof}
Considering a generative view of the prediction model, random variable ${\cal F}$ associated with the learned features is continuous, while the random variable describing the labels, \textit{i.e.}, ${\cal K}$, takes its possible values in a finite set $\{1,\ldots,K\}$.
Then, the region-size distribution of the labels could be empirically estimated by the GT proportion of each segmentation region (as listed in Table 1 of the main text) \citep{KimFYCW05Nonparametric} :
\begin{equation}
    \mathbb{P} ({\cal K} = k) \approx \hat{y}_{k} = \frac{|\mathbf{\Omega}_k|}{|\mathbf{\Omega}|}
\end{equation}
Also, we express the conditional entropy of the learned features as follows :
\begin{align}
\label{appendix:eq:cond-entropy-feature}
    {\cal H}({\cal F}|{\cal K}) &= \sum_k^K \mathbb{P} ({\cal K} = k) {\cal H}({\cal F}|{\cal K} = k) \nonumber\\
    &\approx \frac{1}{|\mathbf{\Omega}|} \sum_k^K |\mathbf{\Omega}_k| {\cal H}({\cal F}|{\cal K} = k)
\end{align}
with each ${\cal H}({\cal F}|{\cal K} = k)$ given by :
\begin{equation}
\label{appendix:eq:cond-entropy}
    {\cal H}({\cal F}|{\cal K} = k) = - \int_{\mathbf{f}^\theta} \mathbb{P}(\mathbf{f}^\theta|{\cal K} = k)\log \mathbb{P}(\mathbf{f}^\theta|{\cal K} = k)d\mathbf{f}^\theta
\end{equation}
Hereafter, for notation simplicity, we omit ${\cal K}$ and use ${\cal H}({\cal F}|k)$ instead of ${\cal H}({\cal F}|{\cal K} = k)$. Also, we use $\mathbb{P}(\mathbf{f}^\theta| k)$ instead of $\mathbb{P}(\mathbf{f}^\theta|{\cal K} = k)$.

To estimate the conditional entropy in Eq. (\ref{appendix:eq:cond-entropy}), let us refer to the following well known Monte-Carlo estimation \citep{Kearns1998,TangMAB19Kernel} :

\textbf{Monte-Carlo estimation.} For any discrete set of points $\mathbf{S} \subset \mathbf{\Omega}$, any function g and any feature embedding $\mathbf{f}$, we have :
\begin{equation}
\label{appendix:eq:monte-carlo}
    \int_{\mathbf{f}} \text{g}(\mathbf{f})\mathbb{P}(\mathbf{f}|\mathbf{S}) \approx \frac{1}{|\mathbf{S}|} \sum_{i\in \mathbf{S}} \text{g}(\mathbf{f}_i)
\end{equation}
where $\mathbf{f}_i$ denotes a feature vector at point $i$, and $\mathbb{P}(\mathbf{f}|\mathbf{S})$ stands for the density of $\{\mathbf{f}_i,i\in \mathbf{S}\}$.

Therefore, applying Montre-Carlo to ${\cal H}({\cal F}|k)$ in Eq. (\ref{appendix:eq:cond-entropy}), we can re-write Eq. (\ref{appendix:eq:cond-entropy-feature}) as follows :
\begin{align}
    {\cal H}({\cal F}|{\cal K}) \approx - \frac{1}{|\mathbf{\Omega}|} \sum_k^K \sum_{i \in \mathbf{\Omega}_k} \log(\mathbb{P}(\mathbf{f}_i^{\theta}|k))
\end{align}

Furthermore, using Bayes rule $\mathbb{P}(\mathbf{f}_i^{\theta}|k) \propto \frac{p_{ik}}{\hat{p}_k}$, in addition to the fact that $\sum_{i \in \mathbf{\Omega}_k} \log(\hat{p}_k) = |\mathbf{\Omega}_k| \log(\hat{p}_k)$, we obtain :
\begin{align}
    {\cal H}({\cal F}|{\cal K}) &\approx - \frac{1}{|\mathbf{\Omega}|} \sum_k^K \sum_{i \in \mathbf{\Omega}_k} \log \left (\frac{p_{ik}}{\hat{p}_k} \right ) \nonumber\\
    &= - \frac{1}{|\mathbf{\Omega}|} \sum_k^K \sum_{i \in \mathbf{\Omega}_k} \log\left(p_{ik}\right) + \frac{1}{|\mathbf{\Omega}|} \sum_k^K \sum_{i \in \mathbf{\Omega}_k} \log\left({\hat{p}_k}\right) \nonumber\\
    &= \text{CE} + \frac{1}{|\mathbf{\Omega}|} \sum_k^K |\mathbf{\Omega}_k| \log(\hat{p}_k) \nonumber\\
    &= \text{CE} + \sum_k^K \hat{y}_k \log(\hat{p}_k)
\end{align}

Finally, due to the definition of the region-size KL divergence, we have :
\begin{align}
{\cal D}_\text{KL}(\mathbf{{y}}||\mathbf{p}) =  \sum_{k=1}^{K} \hat{y}_{k} \log \left ( \frac{\hat{y}_{k}}{\hat{p}_{k}} \right ) \ceq - \sum_k^K \hat{y}_k \log(\hat{p}_k )
\end{align}
This yields : 

\begin{align}
    \text{CE} \ceq \cal{H}(\cal{F}|{\cal{K}}) + {\cal D}_\text{KL}({\mathbf y} || \mathbf{p})
\end{align}

In summary, we give an information-theoretic prospective of CE.
The entropy term can be considered as a ground-truth matching term, while the region-size KL term avoids trivial solutions and encourages the proportions of the predicted segmentation regions to match the ground-truth proportions.

\end{proof}

\section{The binary segmentation case}
\label{appendix:binary}

In the two-class (binary) segmentation case, Dice might be used for the foreground region only \citep{vnet2016}. 
Similarly to the multi-class case discussed in the paper, a single Dice term also decomposes into a ground-truth matching term and region-size penalty, with the latter encouraging extremely imbalanced binary segmentations. For this specific case, the logarithmic Dice and CE could be written as summations over the foreground 
and background segmentation regions:
\begin{equation}
\small
\label{eq:logdice}
-\log(\text{Dice}_1) \ceq \underbrace{-\log\left (\frac{1}{|\mathbf{\Omega}_1|}\sum_{i\in \mathbf{\Omega}_1} p_{i1} \right )}_{\text{Foreground matching: DF}_1} + \underbrace{\log\left ( \sum_{i\in \mathbf{\Omega}} p_{i1} + |\mathbf{\Omega}_1| \right )}_{\text{region-size bias: DB}_1}
\end{equation}
\begin{equation}
\label{eq:bince}
\text{CE} = \underbrace{- \frac{1}{|\mathbf{\Omega}_1|} \sum_{i\in \mathbf{\Omega}_1} \log p_{i1} }_{\text{Foreground matching: $\text{CE}_1$}} - \underbrace{\frac{1}{|\mathbf{\Omega}_2|} \sum_{i\in \mathbf{\Omega}_2} \log (1 - p_{i1}) }_{\text{Background matching: $\text{CE}_2$}}
\end{equation}
In Eq. (\ref{eq:logdice}), the term DB$_1$ can be expressed, up to an additive constant, as a function of the region-size probability of the foreground class ($k=1$) as follows:
\begin{equation}
\label{eq:dicebias-binary}
\text{DB}_1 \ceq \log (\hat{p}_{1} + \hat{y}_{1})  
\end{equation}
Clearly, the region-size probability $\hat{p}_{1}$ measures the predicted proportion of pixels within the foreground region. This term reaches 
its minimum when the foreground region is empty ($p_{i1}=0 \, \forall i$). Therefore, since $\log$ is monotonically increasing, minimizing 
term DB$_1$ in Eq. (\ref{eq:logdice}) introduces a bias preferring small foreground structures. Note that this region-size penalty in the 
logarithmic Dice loss is important to avoid trivial solutions: when using the foreground-matching term alone, the model may assign all 
the pixels in the image to the foreground region. 

The foreground-matching terms, CE$_1$ and DF$_1$, are closely related, with the former being and upper bound on the latter, due to Jensen's inequality: 
DF$_1$ $\leq$ CE$_1$. Both foreground-matching terms are monotonically decreasing functions of each softmax and reach their global minimum when all the softmax
predictions in the ground-truth foreground are equal to $1$ (i.e., reach their target). Hence, the matching terms in Dice and CE can be viewed as 
two different penalty functions for imposing the same equality constraints, $p_{i1} = 1, \forall i \in {\mathbf \Omega}_1$, thereby encouraging the predicted 
foreground to include the ground-truth foreground.

\section{The temperature scaling}
\label{appendix:temp}

\begin{figure}[h!]
    \centering
    \includegraphics[width=0.90\columnwidth]{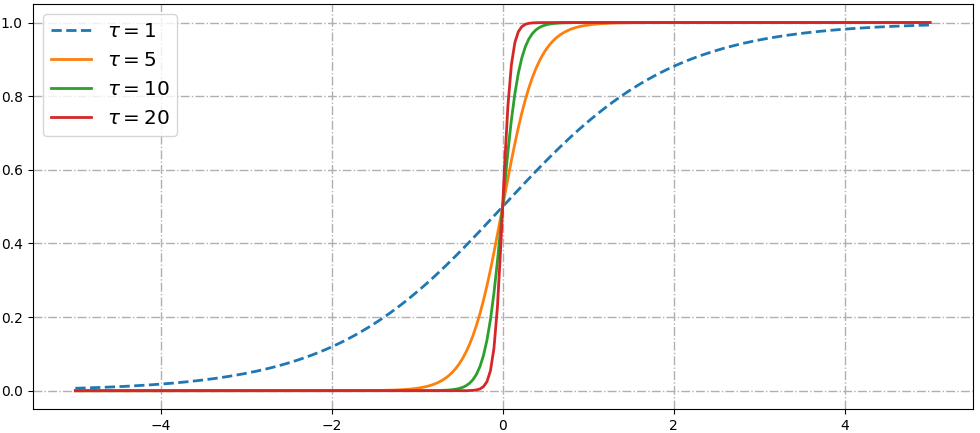}
    \caption{\textbf{Comparison of soft-max functions with different temperature scaling parameter $\tau$.} It is shown that the confidence of the output increases with larger $\tau$. Best seen in color.}
    \label{appendix:fig:temp}
\end{figure}

In our implementation, we employ a modified soft-max function with a temperature scaling parameter when computing the predicted region proportion $\hat{p}_k$ (also referred to as the predicted region-size probability, as in Table 1 of the main text) :
\begin{equation}
    s(\mathbf{z})_i = \frac{e^{\tau \cdot z_i}}{\sum_j^K e^{\tau \cdot z_j}}
\end{equation}
where $\mathbf{z} = (z_i)_{1\leq i \leq K}$ is the input vector of the soft-max function, and $\tau > 0$ acts as the temperature hyper-parameter.
High values of $\tau > 0$ yield high confidence of the soft-max prediction, as shown in Fig. \ref{appendix:fig:temp}: They push the soft-max vector towards the vertices of the simplex, with prediction values approaching either $0$ or $1$.  
As a result, this enables a better estimate of the actual region proportion (or relative size).
Note, we set $\tau$ to $10$ throughout all our experiments.

\rev{
\section{Experiments on natural images}

We further investigate our method on a natural image segmentation benchmark, i.e., Cityscapes\citep{Cordts2016Cityscapes}, in comparison with other baselines.
Cityscapes is a large-scale natural scene dataset with high quality pixel-level annotations of $5$k images across $19$ categories, containing both stuff and objects with high variation in the class proportions distribution.
We use the official data split, which contains $2,975$ samples for training, $500$ samples for validation and $1,512$ samples for testing.
All the input images are resized to $512\times 1024$ for training, and $1024\times 2048$ for testing.
The main setting is adopted from the state-of-the-art library\footnote{\url{https://github.com/open-mmlab/mmsegmentation}}.
Specifically, SGD optimizer is used for training, with the initial learning rate set to $0.01$ and a batch size of $8$.
During the $100$ training epochs, we use an iteration-wise polynomial strategy to linearly scale the learning rate down to a minimum of $1\text{e-}4$.

\newcolumntype{g}{>{\columncolor{Gray}}c}
\begin{table*}[htb]
    \caption{\textbf{Results on Cityscapes validation set.} (Top: \textit{Res50-FPN}, Bottom: \textit{Res101-FPN}) The second row indicates the average region proportion (mRegProp) for each class. mIoU denotes the mean IoU score over all classes.} %
    \vspace{-5mm}
	\begin{center}
	    \resizebox{1.0\textwidth}{!}
		{
		\setlength{\tabcolsep}{2pt}
		\begin{tabular}{@{} l  c c c c c c c c c c c c c c c c c c c  | c   @{}}
			\toprule
			& road & swalk & build. & wall & fence & pole & tlight & sign & veg. & terrain & sky & person & rider & car & truck & bus & train & mbike & bike & \textbf{mIoU}  \\
			\rowcolor{LightCyan}
			mRegProp(\%) & 32.6 & 5.4 & 20.2 & 0.6 & 0.8 & 1.1 & 0.2 & 0.5 & 14.1 & 1.0 & 3.6 & 1.1 & 0.1 & 6.2 & 0.2 & 0.2 & 0.2 & 0.1 & 0.4 &  \\
			\midrule
			CE & 97.2 & \textbf{80.6} & 90.0 & 38.2 & 51.3 & 55.7 & 60.5 & 72.2 & \textbf{91.0} & 58.6 & 92.3 & 75.6 & 48.5 & 92.5 & 50.5 & 67.7 & 56.3 & 50.2 & 71.6 & 68.5  \\
			FL & 97.1 & 78.8 & 89.5 & 37.9 & 48.9 & 51.8 & 56.9 & 67.6 & 90.6 & 56.9 & 92.6 & 73.1 & 42.8 & 92.0 & 47.5 & 57.5 & 46.1 & 49.2 & 68.4 & 65.5  \\
			Dice & 95.9 & 77.6 & 87.7 & \textbf{41.1} & 47.6 & 56.0 & \textbf{66.2} & \textbf{74.8} & 90.0 & 59.6 & 93.1 & \textbf{77.3} & \textbf{55.8} & 91.1 & 10.0 & 63.9 & 34.6 & 53.6 & \textbf{73.6} & 65.8 \\
		    LogDice & 94.0 & 70.8 & 85.6 & 33.2 & 39.1 & 50.6 & 62.2 & 69.6 & 88.4 & 52.1 & 88.1 & 74.2 & 50.7 & 89.2 & 34.8 & 55.3 & 31.7 & 48.0 & 70.7 & 62.5  \\
            DiceCE & 97.0 & 79.7 & 89.7 & 43.4 & 48.6 & 55.4 & 61.4 & 71.7 & 91.0 & 57.4 & 92.6 & 75.5 & 48.6 & 92.5 & 44.4 & 63.6 & 52.1 & 51.8 & 71.5 & 67.8 \\
			LogDiceCE & 96.9 & 78.6 & 89.8 & 40.4 & 48.3 & \textbf{57.1} & 65.4 & 74.6 & 90.7 & 57.2 & 92.5 & 77.0 & 52.7 & 92.4 & 47.9 & 65.0 & 50.6 & 51.4 & 72.6 & 68.5  \\
            \rowcolor{LightGray} RCE(${\cal D}_\text{KL}$) & 97.2 & 79.6 & 90.0 & 39.6 & 51.2 & 54.8 & 60.3 & 71.5 & \textbf{91.0} & \textbf{59.7} & 92.9 & 75.2 & 49.0 & 92.7 & \textbf{55.6} & \textbf{72.5} & \textbf{65.6} & 51.3 & 70.9 & 69.5 \\
            \rowcolor{LightGray} {\bf RCE(${\cal L}_1$)} & \textbf{97.4} & 80.2 & \textbf{90.1} & 37.9 & \textbf{51.9} & 55.7 & 61.0 & 72.1 & \textbf{91.0} & 58.9 & \textbf{93.5} & 75.7 & 49.8 & \textbf{92.8} & 54.9 & 70.2 & 62.8 & \textbf{54.0} & 71.6 & \textbf{69.6}  \\
			\midrule
			CE & \textbf{97.8} & 82.4 & 91.0 & 46.5 & 55.3 & 57.2 & 62.6 & 72.3 & 91.5 & 60.7 & \textbf{94.2} & 76.1 & 50.7 & 93.6 & 68.4 & 77.2 & 64.8 & 54.9 & 72.6 & 72.1  \\
			FL & 97.7 & 81.9 & 90.8 & \textbf{48.2} & 54.1 & 54.4 & 59.0 & 69.4 & 91.0 & 60.2 & 93.7 & 74.7 & 47.8 & 93.0 & 59.3 & 67.3 & 51.9 & 52.1 & 70.5 & 69.3  \\
			Dice & 96.4 & 79.6 & 88.1 & 0.0 & 50.6 & \textbf{58.4} & \textbf{69.0} & \textbf{76.0} & 90.2 & 60.4 & 93.7 & \textbf{78.7} & \textbf{60.0} & 91.4 & 35.1 & 0.0 & 26.8 & 0.0 & \textbf{74.5} & 59.4  \\
			LogDice & 95.2 & 73.9 & 86.4 & 31.5 & 39.2 & 52.4 & 63.3 & 70.4 & 88.6 & 53.0 & 92.2 & 74.6 & 52.1 & 89.4 & 37.1 & 57.3 & 31.1 & 42.7 & 71.5 & 63.3  \\
            DiceCE & 97.6 & 81.8 & 90.9 & 46.7 & 52.8 & 59.5 & 68.5 & 76.8 & 91.3 & 59.5 & 94.0 & 78.3 & 57.1 & 93.5 & 56.9 & 72.1 & 56.6 & 53.2 & 74.5 & 71.7 \\
			LogDiceCE & 97.2 & 79.9 & 90.2 & 42.5 & 52.2 & 57.6 & 66.7 & 74.8 & 90.9 & 59.6 & 93.7 & 77.7 & 57.4 & 92.9 & 56.4 & 72.1 & 58.9 & 54.2 & 74.0 & 71.0  \\
			\rowcolor{LightGray} RCE(${\cal D}_\text{KL}$)  & \textbf{97.8} & 82.6 & 91.1 & 45.4 & \textbf{56.8} & 57.4 & 63.4 & 72.5 & 91.5 & \textbf{61.3} & 94.0 & 76.7 & 52.4 & 93.7 & 69.2 & 78.3 & 64.5 & 57.1 & 72.8 & 72.6  \\
			\rowcolor{LightGray} {\bf RCE(${\cal L}_1$)} & \textbf{97.8} & \textbf{82.9} & \textbf{91.2} & 48.0 & \textbf{56.8} & 57.7 & 63.8 & 72.7 & \textbf{91.6} & 61.2 & 93.8 & 76.9 & 52.8 & \textbf{93.8} & \textbf{77.1} & \textbf{80.0} & \textbf{67.1} & \textbf{57.2} & 73.1 & \textbf{73.4}  \\
			\bottomrule
		\end{tabular}
		}
	\end{center}
	
	\label{tab:cityscapesresult}
\end{table*}

Table \ref{tab:cityscapesresult} reports the comparative per-class IoU and mean IoU (mIoU) on the validation set of Cityscapes with two network architectures. 
First, we can observe that in this multi-class dataset, regardless of the network, the proposed learning objectives outperform all the evaluated losses in terms of mIoU.
Then, by investigating the relationship between region proportion, mRegProp (second row in Table \ref{tab:cityscapesresult}), and the corresponding segmentation performance across small region classes, we can observe that Dice-related losses have a hidden label-marginal bias towards extremely imbalanced solutions, preferring small structures. In particular, the linear Dice often obtains the highest IoU for the smallest structures. 
This bias comes at the cost of less flexibility when dealing with arbitrary class proportions, which is reflected in its poor mIoU (\textit{right column}). 
Quantitative evaluation on the Cityscapes test set is reported in Table \ref{tab:cityscapes-test}. The observations in this table are consistent with those on the validation set, that our method achieves better results on both architectures.

\begin{table}[t]
\begin{center}
\scriptsize
\caption{\textbf{mIoU on Cityscapes test set.} }
\label{tab:cityscapes-test}
{
\begin{tabular}{@{}lcc@{}}\toprule

Loss & Res50-FPN & Res101-FPN \\
\midrule
CE & 67.0 &  69.6  \\ 
FL & 64.5 & 66.8  \\
Dice & 63.7 & 59.4  \\
LogDice & 59.8 & 62.6  \\
DiceCE  & 66.6 &  69.7  \\
LogDiceCE  & 67.1 &  69.6  \\
\midrule
RCE(${\cal D}_\text{KL}$) & \textbf{68.4} &  \textbf{70.1}   \\
{\bf RCE(${\cal L}_1$)} &  \textbf{68.4} & \textbf{70.1}  \\
\bottomrule
\end{tabular}
}
\end{center}
\end{table}

}

\rev{
\section{Performances comparision on AMOS benchmark}

\begin{table*}[htb]
    \caption{\rev{\textbf{Per-class and mean performance comparison on AMOS benchmark.} We report DSC score for each model and mDSC denotes the mean DSC score across all classes.}}
    \label{tab:AMOS-class}
    \begin{center}
    \vspace{-5mm}
    \resizebox{1.0\textwidth}{!}{
    \begin{tabular}{@{}l*{15}{c}|c@{}}
    \toprule
    & \rev{spleen} & \rev{r.kidney} & \rev{l.kidn.} & \rev{gall.blad.} & \rev{esoph.} & \rev{liver} & \rev{stomach} & \rev{arota} & \rev{postcava} & \rev{panc.} & \rev{r.adre.} & \rev{l.adre.} & \rev{duod.} & \rev{bladder} & \rev{prostate} & \rev{mDSC} \\
    \midrule
    \rev{CE} & \rev{96.6} & \rev{\textbf{96.1}} & \rev{\textbf{96.5}} & \rev{77.0} & \rev{77.6} & \rev{\textbf{97.5}} & \rev{90.2} & \rev{\textbf{94.6}} & \rev{89.4} & \rev{84.6} & \rev{63.3} & \rev{62.8} & \rev{77.1} & \rev{85.8} & \rev{81.2} & \rev{84.7} \\
    \rev{FL} & \rev{96.5} & \rev{95.2} & \rev{96.3} & \rev{75.8} & \rev{75.8} & \rev{97.3} & \rev{88.9} & \rev{94.4} & \rev{88.8} & \rev{83.3} & \rev{60.9} & \rev{60.5} & \rev{75.8} & \rev{85.2} & \rev{80.6} & \rev{83.7} \\
    \rev{Dice} & \rev{96.3} & \rev{96.0} & \rev{96.0} & \rev{79.7} & \rev{0} & \rev{97.1} & \rev{90.6} & \rev{94.3} & \rev{90.5} & \rev{\textbf{86.0}} & \rev{0} & \rev{0} & \rev{\textbf{80.5}} & \rev{0} & \rev{0} & \rev{60.5} \\
    \rev{LogDice} & \rev{0} & \rev{94.5} & \rev{95.4} & \rev{0} & \rev{0} & \rev{95.3} & \rev{85.7} & \rev{93.2} & \rev{88.2} & \rev{0} & \rev{0} & \rev{0} & \rev{74.4} & \rev{0} & \rev{0} & \rev{41.8} \\
    \rev{DiceCE} & \rev{96.5} & \rev{96.0} & \rev{96.1} & \rev{77.4} & \rev{78.1} & \rev{97.4} & \rev{90.0} & \rev{94.5} & \rev{89.5} & \rev{84.2} & \rev{62.9} & \rev{63.8} & \rev{77.3} & \rev{85.1} & \rev{80.7} & \rev{84.6} \\
    \rev{LogDiceCE} & \rev{96.1} & \rev{\textbf{96.1}} & \rev{96.2} & \rev{79.6} & \rev{83.0} & \rev{97.3} & \rev{90.2} & \rev{94.4} & \rev{90.4} & \rev{85.3} & \rev{74.1} & \rev{75.1} & \rev{79.4} & \rev{84.0} & \rev{\textbf{81.5}} & \rev{86.9} \\
    \midrule
    \rev{\textbf{RCE (Ours)}} & \rev{\textbf{96.7}} & \rev{\textbf{96.1}} & \rev{96.3} & \rev{\textbf{79.9}} & \rev{\textbf{83.3}} & \rev{\textbf{97.5}} & \rev{\textbf{90.7}} & \rev{\textbf{94.6}} & \rev{\textbf{90.8}} & \rev{85.7} & \rev{\textbf{75.1}} & \rev{\textbf{75.9}} & \rev{79.6} & \rev{\textbf{86.2}} & \rev{80.8} & \rev{\textbf{87.3}} \\
    \bottomrule
    \end{tabular}
    }
    \end{center}
\end{table*}

Table~\ref{tab:AMOS-class} presents the comparison of per-class performances for different methods.
Across the 15 categories, our method outperforms the related works on 11 classes, and achieve the best performance in term of mean DSC.
It is noteworthy that singular Dice-related losses, i.e., Dice and LogDice, achieve the best performances on some classes (86.0 for pancreas and 80.5 for duodenum), but  completely failed to handle some classes, like esophagus, bladder and prostate.
This finding provides further support for the theoretical insights presented in our paper, namely that while Dice loss encourages segmentation of classes with small region areas, it might not be as effective in adapting to different classes or instances with diverse region proportions.
Another possible reason could be the aggressive gradient nature of Dice loss at the vicinity, which can lead to instability in the optimization process.
}

\rev{
\section{Discussion regarding computation complexity and training converge}

\begin{figure}[htb]
    \centering
    \includegraphics[width=0.9\columnwidth]{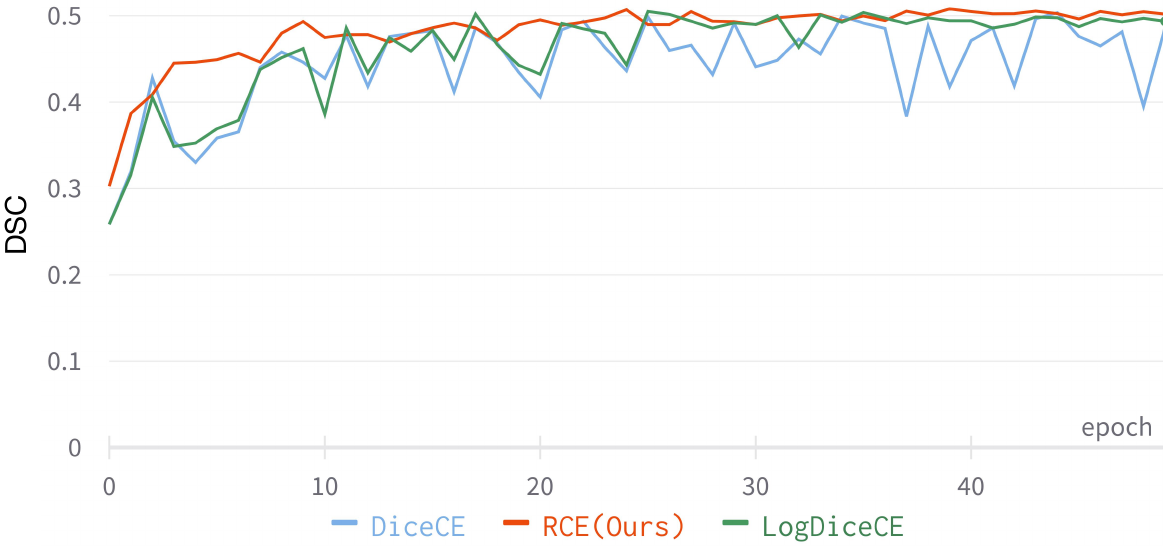}
    \caption{\rev{Comparison of training evolution for different losses. We present the evolution of DSC scores on validation set of Retinal lesions during training.}}
    \label{appendix:fig:train}
\end{figure}

Comparing to the other related two-term compounding losses like DiceCE and LogDiceCE, the proposed method share similar computation complexity.
Thus, the influence on the training speed is marginal.
As the imposed penalty term encourages to predict right region proportion, which is reinforcing the hidden bias inside CE, it does not affect the convergence of the training process.
Empirically, we investigate the training progress of our method by presenting the DSC scores on validation set of Retinal Lesions over the training epoches, in comparison with DiceCE and LogDiceCE, as shown in Fig.~\ref{appendix:fig:train}.
It is shown that our method deliver more stable evolution during the training process, which might be due to its better gradient dynamics. 
}

\end{document}